%% file: main.tex
\documentclass[11pt]{article}
\usepackage[margin=1in]{geometry}
\usepackage[blocks]{authblk}
\makeatletter
\@ifundefined{AB@maketitle}{%
  \newcommand{\authorbox}[1]{#1}%
}{%
  \let\maketitle\AB@maketitle
  \newcommand{\authorbox}[1]{\makebox[0.38\textwidth][c]{#1}}%
}
\makeatother

\input{math_commands.tex}

\DeclareMathOperator{\Pa}{Pa}
\DeclareMathOperator{\MSE}{MSE}
\DeclareMathOperator{\depth}{depth}
\DeclareMathOperator{\spn}{span}
\DeclareMathOperator{\adj}{adj}
\DeclareMathOperator{\Real}{Re}
\newcommand{\ip}[2]{\langle #1,#2\rangle}
\newcommand{\nrm}[1]{\lVert #1\rVert}
\newcommand{\siR}[3]{R_{#1}(#2,#3)}
\newcommand{\siBR}[3]{\overline R_{#1}(#2,#3)}

\usepackage{amsthm}
\usepackage{hyperref}
\usepackage{url}
\usepackage{csquotes}
\usepackage[style=alphabetic,natbib=true,maxalphanames=3,minalphanames=3,maxbibnames=99]{biblatex}
\usepackage{thmtools}
\makeatletter
\@ifundefined{newcounteralias}{}{%
  \renewcommand\thmt@autorefsetup{%
    \@xa\def\csname\thmt@envname autorefname\@xa\endcsname
      \@xa{\thmt@thmname}%
  }%
}
\IfFormatAtLeastTF{2024-11-01}{
  \renewcommand*\@addtoreset[2]{%
    \bgroup
      \edef\aliasctr@@truelist{\aliasctr@follow{#2}}%
      \let\@elt\relax
      \expandafter\@cons\aliasctr@@truelist{{#1}}%
    \egroup
    \expandafter\xdef\csname theH#1\endcsname{%
      \expandafter\noexpand\csname theH#2\endcsname.%
      \noexpand\the\noexpand\value{#1}}%
  }
}{}
\makeatother
\usepackage{thm-restate}
\usepackage{cleveref}
\usepackage{amssymb}
\usepackage{booktabs}
\usepackage{adjustbox}
\usepackage{etoolbox}
\usepackage{float}

\makeatletter
\expandafter\patchcmd\csname\string\thmt@restatable\endcsname
  {\thmt@trivialref{thmt@@#3}{??}}
  {\protect\ref{thmt@@#3}}
  {}
  {\PackageError{restate-links}{Could not patch thm-restate}
    {Check whether the installed thm-restate version has changed.}}
\makeatother

\theoremstyle{plain}

\newtheorem{theorem}{Theorem}[section]

\newtheorem{lemma}[theorem]{Lemma}

\newtheorem{proposition}[theorem]{Proposition}

\theoremstyle{definition}

\newtheorem{definition}[theorem]{Definition}

\AddToHook{cmd/appendix/before}{\crefalias{section}{appendix} \crefalias{subsection}{appendix}}

\title{Optimal Networks for Agentic Information Aggregation}
\author{\authorbox{MohammadHossein~Bateni}}
\affil{Google Research\\\texttt{bateni@google.com}}
\author{\authorbox{Zahra~Hadizadeh}}
\affil{University of California, Irvine\\\texttt{zhadizad@uci.edu}}
\author{\authorbox{MohammadTaghi~Hajiaghayi}}
\affil{University of Maryland\\\texttt{hajiagha@umd.edu}}
\author{\authorbox{Mahdi~JafariRaviz}}
\affil{University of Maryland\\\texttt{mahdij@umd.edu}}
\author{\authorbox{Shayan~Taherijam}}
\affil{University of California, Irvine\\\texttt{staherij@uci.edu}}
\date{}

\begin{document}
\maketitle

\begin{abstract}
We study information aggregation in the networked learning model introduced by \citeauthor{kearns2026networked} (SODA~\citeyear{kearns2026networked}). There is a fixed distribution over $d$ features and a common label. Agents learn in topological order on a directed acyclic graph. Each observes a subset of the features and its parents' predictions, fits a linear predictor to minimize mean squared error, and passes only its prediction forward. The global predictor is the best linear predictor using all features. \citeauthor{kearns2026networked} show that the output agent's error approaches the global predictor's error along sufficiently deep paths with suitable feature coverage, while insufficient depth can prevent aggregation even in large networks.

In contrast to their main focus on a given graph and feature allocation, we consider the limits of the model under two settings. In the \emph{adaptive designer} setting, a designer chooses the graph, feature allocation, and output agent knowing the distribution. In the \emph{oblivious designer} setting, the designer fixes all three before an adversary chooses the distribution. Each agent observes one feature and receives predictions from a limited number of parents.

We call the aggregation \emph{exact} when the output agent matches the global predictor exactly. For $d\ge3$, we show that no finite depth guarantees exact aggregation for every distribution with one parent per agent, even when the designer knows the distribution.

In contrast, two parents per agent suffice for exact aggregation even in the oblivious designer setting. A fixed graph, feature allocation, and output agent achieve this for every distribution at depth $O(d\log d)$. Knowing the distribution reduces the depth to $O(d)$. Both constructions use $O(d^2)$ agents, with a very large constant for two parents. We show the bounds on the depth and number of agents are all optimal up to constant factors.
\end{abstract}

\section{Introduction}
\label{sec:introduction}

Social learning in networks studies how parties with different information learn from one another \citep{degroot1974reaching}. Parties typically learn from the opinions or predictions of others rather than from their private observations. A similar structure appears in multi-agent AI systems, where a task is split across several models and later models build on the outputs of earlier ones \citep{guo2024multiagent}. In such systems the network is a design choice: the designer decides which model sees which information and which models talk to each other. This raises a basic question. If each party passes on only its prediction, can a well-designed network still do as well as a single learner that sees everything? And if so, how many predictions must each learner receive, and how deep and how large must the network be?

\citet{kearns2026networked} model this kind of distributed learning as a network of agents, where each agent has access to a subset of the features and learns a model to predict a common label. Agents learn in turn and then forward their predictions to their successors in the network. Each prediction is therefore both an estimate of the label and the summary of an agent's information that later agents receive. They ask how the output agent's prediction compares to a global predictor that has access to all the features.

More formally, let $G=(A,E)$ be a directed acyclic graph (DAG), and let $\gD$ be a distribution over $(x,Y)$, with $x=(x_1,\ldots,x_d)$ being the vector of features and $Y$ being the label. Each agent has direct access to a subset of the features of $x$ and receives the predictions of its parents in $G$. Following a topological order, agents fit linear predictors of $Y$ from these inputs to minimize mean squared error (MSE). The \emph{excess error} of a prediction is the difference between its MSE and the MSE of the best linear global predictor using all $d$ features. We say that a network achieves \emph{information aggregation} when the output agent's prediction is competitive with the global predictor, meaning that its excess error is small. We call the aggregation \emph{exact} when this excess error is zero.

One might expect that observing every feature somewhere in the network would be enough for exact aggregation. However, a feature that does not help predict the label on its own may become useful when combined with another feature. An agent may leave such a feature out of its prediction, so later agents receive no information about it.

\citet{kearns2026networked} ask under what conditions on the graph and feature allocation the network can achieve information aggregation. Their main results focus on analyzing an instance of the problem where the distribution, the graph, the feature allocation, and the output agent are already chosen. They show that under a certain condition on the graph and feature allocation, the excess error of the output agent converges to zero as the depth $D$ grows. Here, depth is the number of agents on a longest path ending at the output agent. They also give examples of distributions, graphs, and feature allocations for which the excess error is bounded below by an inverse polynomial in $D$.

They also show that depth can be necessary even when the graph and feature allocation are chosen in the best possible way for the distribution. Specifically, \citet[Theorem~5.9]{kearns2026networked} give, for every $d$, a distribution on $d$ features such that every DAG of depth at most $D<d$ and every allocation of one feature per agent have excess error at least $1/(D+1)$ at the output agent. We emphasize the restriction $D<d$ in this result, which leaves open what is possible when the depth reaches or exceeds the number of features.

To understand the limits of the model from a practical perspective, in this work we consider two natural settings for choosing the instance. In the \emph{adaptive designer} setting, an adversary first chooses a distribution $\gD$ over $(x,Y)$ with $d$ features. A designer then chooses a DAG $G$ of depth at most $D$, an allocation of one feature per agent, and an output agent with knowledge of $\gD$. In the \emph{oblivious designer} setting, the designer first chooses the DAG $G$ of depth at most $D$, an allocation of one feature per agent, and the output agent without knowledge of $\gD$. The adversary then chooses $\gD$ with knowledge of the designer's choices, so the same graph, allocation, and output agent must work for every distribution.

We also require each agent to receive predictions from at most $b$ parents in both settings. Without this restriction, a single agent could be asked to fit a model from many parent predictions, making its learning problem as large as the global one. With this bound, each agent learns from at most $b+1$ inputs, including its raw feature.

In both settings, we ask how quickly excess error can decrease with depth, and how much depth and how many agents are necessary and sufficient for exact aggregation. We denote the optimal worst-case excess errors at depth $D$ by $\siR{b}{d}{D}$ in the adaptive designer setting and $\siBR{b}{d}{D}$ in the oblivious designer setting. We define these quantities formally in \Cref{sec:preliminaries}.

\subsection{Our results}
\label{subsec:our-results}

We summarize our main results in \Cref{tab:main}. For the excess-error bounds, we bound each feature's second moment and the sum of the $\ell_1$ norm of the global predictor's coefficients by one (\Cref{def:normalized-distribution}). The constructions for exact aggregation require only finite second moments.

\begin{table}[htbp]
\centering
\caption{Bounds for $d\ge3$ features and fixed parent limit $b$.}
\label{tab:main}
\renewcommand{\arraystretch}{1.15}
\begin{adjustbox}{width=\linewidth}
\begin{tabular}{clcc}
\toprule
\shortstack{Parent limit ($b$)} & Quantity & Adaptive designer & Oblivious designer \\
\midrule
$1$ & Excess error at depth $D$ & $\siR{b}{d}{D}=\Theta(1/D)$ & $\siBR{b}{d}{D}=\Omega(1/D)$ \\
\addlinespace
$2$ & Depth for exact aggregation & $\Theta(d)$ & $\Theta(d\log d)$ \\
\addlinespace
$\ge3$ & Depth for exact aggregation & $d$ & $\Theta(d\log d)$ \\
\addlinespace
$\ge2$ & Number of agents for exact aggregation & $\Theta(d^2)$ & $\Theta(d^2)$ \\
\bottomrule
\end{tabular}
\end{adjustbox}
\end{table}

We first consider $b=1$, where the ancestors of the output form a path. For every $d\ge3$ and $D\ge1$, we show that $\siR{1}{d}{D}=\Theta(1/D)$ (\Cref{cor:b-1:rate}). This also gives $\siBR{1}{d}{D}=\Omega(1/D)$ in the oblivious designer setting. The matching upper bound in the adaptive designer setting uses knowledge of the distribution to choose the feature allocation along a path of $D$ agents in a greedy way. An oblivious designer cannot make these greedy choices. A natural fixed choice is a path whose agents observe $x_1,\ldots,x_d$ in cyclic order, so that every $d$ consecutive agents see all features. \citet[Theorem~1]{bateni2026optimal} show that such a path has error $O(d^2/D)$. Thus $\siBR{1}{d}{D}$ lies between $\Omega(1/D)$ and $O(d^2/D)$, and we leave closing this gap in $d$ open.

In contrast to $b=1$, with three parents per agent ($b=3$), the designer can fix a graph, allocation, and output agent that achieve exact aggregation for every distribution at depth $O(d\log d)$ using $O(d^2)$ agents (\Cref{thm:b-3:fixed}). When the designer knows the distribution, we reduce the worst-case depth to $d$ while still using $O(d^2)$ agents (\Cref{thm:b-3:exact}).

We show that each three-parent agent can be replaced by a fixed gadget with a constant number of two-parent agents that reproduces its prediction for every distribution (\Cref{lem:b-2:fixed-gadget}). This constant is independent of $d$ and of the distribution, but it is very large, since the gadget runs all small two-parent networks in parallel. Applying this replacement gives exact aggregation with $b=2$, depth $O(d)$ in the adaptive designer setting and $O(d\log d)$ in the oblivious designer setting, using $O(d^2)$ agents in both (\Cref{thm:b-2:bounds}).

We then show that these depth bounds are optimal up to constant factors (\Cref{thm:b-3:depth-lower}). In the adaptive designer setting, some distributions require depth at least $d$ for exact aggregation, even with no parent limit. Thus the depth-$d$ bound with three parents is exactly optimal in the worst case. In the oblivious designer setting, every fixed parent limit $b\ge2$ requires depth $\Omega(d\log d)$, matching the construction for exact aggregation.

Finally, we show that the quadratic number of agents in our constructions is necessary. For every fixed $b\ge2$, some normalized distribution requires $\Omega(d^2)$ agents for exact aggregation, regardless of depth and even when the designer knows the distribution (\Cref{thm:size:lower-bound}). This lower bound therefore also holds in the oblivious designer setting, so the number of agents is optimal in both settings.

\section{Preliminaries}
\label{sec:preliminaries}

Let $(x,Y)\sim\gD$, where $x=(x_1,\ldots,x_d)\in\R^d$ is the feature vector and $Y\in\R$ is the label. We assume that the features and the label have finite second moments. We write $[d]=\{1,\ldots,d\}$, and all expectations are over $\gD$.

We work in $L^2(\gD)$, the space of real random variables with finite second moment, and regard two variables as equal if they agree with probability one. For $U,V\in L^2(\gD)$, we write $\ip{U}{V}=\E[UV]$ and $\nrm{U}^2=\E[U^2]$. We use $f$ for both a predictor and the random variable $f(x)$. Its mean squared error is then $\MSE(f)=\E[(Y-f)^2]=\nrm{Y-f}^2$. For subspaces $V,W\subseteq L^2(\gD)$, their sum $V+W=\{v+w\mid v\in V,\ w\in W\}$ is the smallest subspace containing both.

We call vectors $u,v\in L^2(\gD)$ orthogonal, written $u\perp v$, when $\ip{u}{v}=0$. We write $u\perp V$ when $u$ is orthogonal to every vector in $V$, and $V\perp W$ when every vector in $V$ is orthogonal to every vector in $W$. The orthogonal complement $V^\perp$ is the subspace of all vectors in $L^2(\gD)$ orthogonal to $V$.

There are agents $A=\{A_1,\ldots,A_n\}$ in a DAG $G=(A,E)$ with a designated output agent $A_G$. An edge $(A_j,A_i)\in E$ means that $A_i$ receives the prediction of $A_j$. We write $\Pa(A_i)=\{A_j\in A\mid (A_j,A_i)\in E\}$ for the set of its parents. Agent $A_i$ sees the features $x_{S_i}$ for a set $S_i\subseteq[d]$, together with its parents' predictions, and fits the best linear predictor from these inputs. Agents fit in a topological order, so every parent has been fitted before its children.

Let $f_i$ be the predictor of agent $A_i$. The linear combinations of its inputs form the space
\begin{equation*}
 V_i=\spn\left(\{x_\ell\mid \ell\in S_i\}\cup\{f_j\mid A_j\in\Pa(A_i)\}\right).
\end{equation*}
The agent therefore chooses
\begin{equation}
 \label{eq:protocol}
 f_i=\argmin_{f\in V_i}\MSE(f).
\end{equation}
We call this prediction the \emph{fit} of $Y$ from the agent's inputs. We write $f_G$ for the output agent $A_G$'s prediction.

For a finite-dimensional subspace $V\subseteq L^2(\gD)$ and a variable $Z\in L^2(\gD)$, we write $P_VZ$ for the vector in $V$ that minimizes $\nrm{Z-v}^2$ over $v\in V$. This is the orthogonal projection of $Z$ onto $V$. It is characterized by the condition that the residual $Z-P_VZ$ is orthogonal to $V$. The projected vector is unique, even when its coefficients in a given set of inputs are not. Projection is also linear in $Z$. In this notation, the fit in \eqref{eq:protocol} is $f_i=P_{V_i}Y$. 
For finite-dimensional orthogonal subspaces $V,W\subseteq L^2(\gD)$, projection onto their sum splits as $P_{V+W}Z=P_VZ+P_WZ$ for every $Z\in L^2(\gD)$.

The global predictor fits over $H=\spn\{x_1,\ldots,x_d\}$. We write $f^*=P_HY$ for its prediction and $r=\dim H$ for the feature rank. Every agent's prediction lies in $H$, since its raw features lie in $H$ and, by induction, so do its parents' predictions. For any predictor $f\in H$, its excess error is $\MSE(f)-\MSE(f^*)$. We say that the network achieves \emph{exact aggregation} when $f_G=f^*$.

\subsection{Projection}
The following identity relates the improvement in mean squared error to the change in the prediction. We will use it both to compare an agent with the global predictor and to track error along the graph. The proof is in \Cref{app:preliminaries}.

\begin{restatable}{lemma}{leastsquareserror}
\label{lem:least-squares-error}
Let $V\subseteq L^2(\gD)$ be a finite-dimensional subspace and let $f=P_VY$. Then $\ip{Y}{f}=\nrm{f}^2$, and for every $g\in V$,
\begin{equation}
 \label{eq:excess-error}
 \MSE(g)-\MSE(f)=\nrm{g-f}^2.
\end{equation}
\end{restatable}

Applying \Cref{lem:least-squares-error} with $V=H$ and $f=f^*$ gives $\MSE(f_i)-\MSE(f^*)=\nrm{f_i-f^*}^2$. Thus an agent's excess error is its squared distance from the global prediction. For an edge $(A_j,A_i)\in E$, the parent prediction $f_j$ belongs to $V_i$. Applying the lemma with $V=V_i$ and $f=f_i$ gives $\MSE(f_j)-\MSE(f_i)=\nrm{f_j-f_i}^2$. Error is therefore non-increasing along every edge. By \eqref{eq:excess-error}, exact aggregation $f_G=f^*$ is equivalent to having zero excess error.

The part of the label orthogonal to the raw features cannot affect any agent's fit. The next lemma lets us remove that part when analyzing a network. The proof is in \Cref{app:preliminaries}.

\begin{restatable}{lemma}{replaceyfstar}
\label{lem:replace-y-fstar}
Replacing $Y$ by $f^*$ leaves every agent's prediction unchanged.
\end{restatable}

\subsection{Graph constraints and the two settings}

We study networks in which each agent sees one raw feature. An allocation $a\colon A\to[d]$ specifies $S_i=\{a(A_i)\}$. Different agents may observe the same feature. Agents with no path to the output can be deleted without changing its prediction.

The depth $\depth(A_i)$ of agent $A_i$ is the number of agents on a longest directed path ending at $A_i$, so a source has depth one. The depth of the network $\depth(G)$ is $\depth(A_G)$, the depth of its output agent. After deleting agents with no path to $A_G$, this is also the maximum depth in $G$. We write $\Delta^-(G)=\max_{A_i\in A}|\Pa(A_i)|$ for its maximum in-degree. Edges may skip depths, and an agent may send its prediction to any number of children.

For bounds on excess error, we must also fix the scale of the distribution. Otherwise, multiplying the label by a constant can make any positive excess error arbitrarily large. We bound the feature second moments and the sum of the absolute values of the global predictor's coefficients.

\begin{definition}[Normalized distribution]
\label{def:normalized-distribution}
Fix constants $M_X,A^*>0$. A distribution $\gD$ is \emph{normalized} at these bounds if $\nrm{x_i}^2\le M_X^2$ for every $i\in[d]$ and there is a coefficient vector $w^*\in\R^d$ such that $f^*=\sum_{i=1}^d w_i^*x_i$ and $\sum_{i=1}^d|w_i^*|\le A^*$. We write $\gC_d$ for the class of distributions with $M_X=A^*=1$.
\end{definition}

Dividing the features by $M_X$ and the label by $A^*M_X$ reduces these bounds to $M_X=A^*=1$. Thus excess-error bounds for the unit case are multiplied by $(A^*M_X)^2$ at the original scale. Our exact-aggregation results require only finite second moments and do not require normalization.

For positive integers $b,d,D$, we take the infimum over finite DAGs $G=(A,E)$, feature allocations $a\colon A\to[d]$, and output agents $A_G\in A$, subject to $\Delta^-(G)\le b$ and $\depth(A_G)\le D$. We define
\begin{align}
 \siR{b}{d}{D}
 &=\sup_{\gD\in\gC_d}
   \inf_{\substack{G,a,A_G\\
                    \Delta^-(G)\le b,\ \depth(A_G)\le D}}
   \nrm{f^*-f_G}^2,\label{eq:R}\\
 \siBR{b}{d}{D}
 &=\inf_{\substack{G,a,A_G\\
                    \Delta^-(G)\le b,\ \depth(A_G)\le D}}
   \sup_{\gD\in\gC_d}\nrm{f^*-f_G}^2.
 \label{eq:barR}
\end{align}
In the adaptive designer setting, \eqref{eq:R} chooses the best graph, allocation, and output agent for each distribution, then takes the worst error over distributions. In the oblivious designer setting, \eqref{eq:barR} first takes the worst error over distributions for each fixed graph, allocation, and output agent, then minimizes over these choices. Agents fit their coefficients from $\gD$ in both settings.

\begin{restatable}{proposition}{minimaxrelations}
\label{prop:minimax}
For all positive integers $b,d,D$, we have $0\le\siR{b}{d}{D}\le\siBR{b}{d}{D}\le1$. Both quantities are non-increasing in $b$ and $D$, and non-decreasing in $d$.
\end{restatable}

The proof is in \Cref{app:preliminaries}. After deleting agents with no path to the output, there are only finitely many DAGs and allocations under these constraints, up to relabeling. Consequently, $\siR{b}{d}{D}=0$ means that, after seeing any distribution, the designer can choose a graph, allocation, and output agent that achieve exact aggregation within these bounds. For $\siBR{b}{d}{D}=0$, the designer can fix these choices before seeing the distribution and achieve exact aggregation for every distribution.

\section{Related Work}
\label{sec:related-work}

\citet{kearns2026networked} introduce the networked information aggregation model. They prove an $O(M/\sqrt D)$ excess-MSE bound on a path of depth $D$ when every $M$ consecutive agents collectively observe all features, under bounds on feature second moments and the global predictor's coefficient $\ell_1$ norm. \citet{bateni2026networked} extend the protocol to binary classification, where agents minimize binary cross-entropy and pass logits, and prove an $O(M/\sqrt D)$ excess-loss bound under the same coverage condition. \citet{pal2026optimal} sharpens the lower bound for cyclic feature allocations and extends it to a class of convex losses, including logistic loss. \citet{bateni2026optimal} determine the optimal worst-case covered-path rate under fixed moment and coefficient bounds: excess error can remain constant through depth of order $M^2$, and the optimal rate beyond that scale is $\Theta(M^2/D)$. They obtain analogous bounds for logistic classification. These rate results concern given networks and feature allocations under certain conditions. We instead optimize over the graph, feature allocation, and output agent, and study the resulting worst-case excess error as a function of depth and the number of allowed parents.

The closest result to ours is \citet[Theorem~5.9]{kearns2026networked}, discussed in \Cref{sec:introduction}, which applies only to depth $D<d$. An extended discussion of other related work is in \Cref{app:extended-related-work}.

\section{One parent per agent}
\label{sec:one-parent}

With at most one parent per agent, deleting agents with no path to the output leaves a path $A_1,\ldots,A_n$ with $n\le D$.

\citet[Theorem~5.9]{kearns2026networked} give a $1/(D+1)$ excess-error lower bound for every graph and single-feature allocation, but require $D<d$ and use an unnormalized distribution. For every depth $D$, we give a normalized three-feature distribution with error $\Omega(1/D)$ for every path allocation. Knowing the distribution lets the designer achieve a matching upper bound.

\begin{restatable}{theorem}{pathrate}
\label{cor:b-1:rate}
For every $d\ge3$ and $D\ge1$,
\begin{equation*}
 \frac1{640D}\le\siR{1}{d}{D}\le\frac1{D+1},\qquad
 \siBR{1}{d}{D}\ge\frac1{640D}.
\end{equation*}
In particular, $\siR{1}{d}{D}=\Theta(1/D)$.
\end{restatable}

Thus no finite depth guarantees exact aggregation for every distribution with $d\ge3$ features, even when the designer knows the distribution. We sketch both bounds below. The full proofs are in \Cref{app:one-parent}, including the cases of one or two features in \Cref{app:b-1:two-features}.

\subsection{Lower bound}

The features in our construction share a large common component and have small informative components. Similar geometry is used in the covered-path lower bound of \citet{bateni2026optimal}.

Fix an integer $D\ge1$, choose $\rho>0$ with $\rho^2=1/(40D)$, and let $U,V,Z$ be independent standard Gaussians. Define
\begin{equation}
\label{eq:b-1:lower-bound-distribution}
 x_1=Z+\rho U,\qquad
 x_2=Z+\rho V,\qquad
 x_3=Z-\rho(U+V),\qquad
 Y=\rho U.
\end{equation}
Together, the features recover $Y=(2x_1-x_2-x_3)/3$, so $f^*=Y$. A constant rescaling gives a normalized distribution, as shown in \Cref{app:b-1:rate}.

We show that an arbitrary graph with depth at most $D$ has error at least $1/(320D)$ under this distribution. We reduce to the case where the path reaches error $\Omega(\rho^2)$ and every later agent retains at least a $1-10\rho^2$ fraction of its parent's error, as shown in \Cref{app:b-1:lower-bound}. Bernoulli's inequality and our choice of $\rho$ ensure that a constant fraction remains after at most $D$ further steps, giving error $\Omega(1/D)$.

\subsection{Upper bound in the adaptive designer setting}

We give a greedy allocation that selects the feature most correlated with the remaining error, relative to its norm. On a path $A_1,\ldots,A_D$, set $f_0=0$ and, for $t=0,\ldots,D-1$, choose
\begin{equation}
\label{eq:b-1:greedy-allocation}
 a(A_{t+1})\in
 \argmax_{i\in[d]:\,\nrm{x_i}>0}
 \frac{|\ip{Y-f_t}{x_i}|}{\nrm{x_i}}.
\end{equation}
Break ties arbitrarily. Agent $A_{t+1}$ uses this feature and its parent's prediction $f_t$, with no parent when $t=0$.

For distributions in $\gC_d$, write $e_t=\nrm{f^*-f_t}^2$. The coefficient and feature bounds give $e_0\le1$ and ensure that the selected feature has enough correlation with the residual. Adding a multiple of this feature to $f_t$ gives the decrease proved in \Cref{lem:b-1:upper-bound-step}:
\begin{equation*}
 e_{t+1}\le e_t-e_t^2.
\end{equation*}
This recurrence gives $e_D\le1/(D+1)$, as proved for general normalization bounds in \Cref{app:b-1:upper-bound}.

\section{Three parents per agent}
\label{sec:three-parents}

We now show that three parents per agent suffice for exact aggregation in both settings, and we find the optimal depth in each. Throughout this section, we use $f^*$ as the label, as permitted by \Cref{lem:replace-y-fstar}. No normalization is needed for the constructions.

\subsection{The oblivious designer setting}
\label{subsec:b-3:fixed}

The designer must choose the same graph and allocation for every distribution. We organize this graph into rounds. Each round first finds a better prediction, if the current prediction is not exact. It then combines that prediction with the predictions from earlier rounds. This second step ensures that progress in one round is preserved in all later rounds.

\begin{restatable}{theorem}{fixedthreeparents}
\label{thm:b-3:fixed}
For every $d\ge2$, the designer can fix a graph, a single-feature allocation, and an output agent that achieve exact aggregation for every distribution with finite second moments, using at most $4d^2$ agents, at most three parents per agent, and depth $O(d\log d)$.
For $d=1$, one agent suffices.
\end{restatable}

In particular, $\siBR{b}{d}{D}=0$ for $b\ge3$ once $D$ reaches this bound. The full proof is in \Cref{app:b-3:fixed} and below we give a sketch of the construction.

We now describe the construction. Start with a source observing the fixed feature $x_1$, whose prediction is $p_0=P_{\spn\{x_1\}}f^*$. Let $V_0=\spn\{x_1\}$. We build the rest of the graph one round at a time. At round $t+1$, we find a vector $q_t$ that improves on the current prediction $p_t$ if it is not already exactly $f^*$. This would mean that $q_t \notin V_t$ and thus contains a new direction. We then set $V_{t+1} = V_t + \spn\{q_t\}$ and compute the prediction $p_{t+1}=P_{V_{t+1}}f^*$. We will show that the following invariant holds for every $0\le t<d-1$:
\begin{equation}
\label{eq:b-3:invariant}
 V_{t+1}=\spn\{x_1,p_0,\ldots,p_t,q_t\}=\spn\{x_1,p_0,\ldots,p_t,p_{t+1}\}.
\end{equation}
Since the remaining $d-1$ features span $H$ together with $V_0$, the space can grow at most $d-1$ times. It grows whenever $p_t\ne f^*$, so $p_{d-1}=f^*$, as proved in \Cref{app:b-3:fixed}.

To find a new direction $q_t$, we look for an improvement using each raw feature. For every $i=2,\ldots,d$, add an agent observing $x_i$ and receiving $p_t$. If $p_t\ne f^*$, the residual $f^*-p_t$ is a nonzero vector in $H$, so it has a nonzero inner product with some raw feature. That feature cannot be $x_1$, since the residual is orthogonal to $V_t$. At least one of these agents therefore strictly improves on $p_t$.

To collect this improvement in a single prediction, feed all these agents' predictions, together with $p_t$, into a fixed balanced binary tree. We denote the prediction of the root of this tree as $q_t$. Each internal agent observes $x_1$ and receives its two child predictions and $p_t$. By \Cref{lem:least-squares-error}, its error is no larger than either child's error, so the root prediction $q_t$ keeps any improvement found at the leaves. The tree uses $O(d)$ new agents and adds $O(\log d)$ to the depth. Since $p_t$ is already the best predictor in $V_t$, a strict improvement also proves that $q_t\notin V_t$.

It remains to add this direction to $V_{t}$ to get $V_{t+1}$. We want to compute
\begin{equation}
\label{eq:b-3:correction}
 p_{t+1}=P_{V_t+\spn\{q_t\}}f^*.
\end{equation}
By \eqref{eq:b-3:invariant} and the initialization, we have $V_t+\spn\{q_t\} = \spn\{x_1,p_0,\ldots,p_t,q_t\}$. We use a gadget that combines these vectors through a balanced tree, with at most three parents per agent.

The gadget recursively splits $p_0,\ldots,p_t$ into two nearly equal intervals sharing a boundary prediction. Each interval computes the best prediction from its entries and $x_1,p_t,q_t$. An agent observing $x_1$ merges the two interval predictions with a third fitted from $x_1,p_t,q_t$ and their shared boundary. The nested projections $p_i=P_{V_i}f^*$ make this merge exact by \Cref{lem:b-3:history-merge}. The recursion stops at pairs $p_i,p_{i+1}$, handled by \Cref{lem:b-3:history-leaf}. The gadget depends only on $t$, and its added agents observe $x_1$ and have at most three parents. It adds $O(t)$ agents and $O(\log t)$ depth (\Cref{lem:b-3:history-extension} in \Cref{app:b-3:fixed}).

Applying this gadget gives \eqref{eq:b-3:correction} and preserves \eqref{eq:b-3:invariant}, as proved in \Cref{app:b-3:fixed}. We run $d-1$ rounds and use $p_{d-1}$ as the output. Since each round uses $O(d)$ agents and adds $O(\log d)$ to the depth, the total number of agents is $O(d^2)$ and the total depth is $O(d\log d)$.

\subsection{The adaptive designer setting}
\label{subsec:b-3:adaptive}

When the designer knows the distribution, it can choose which feature to use next. This reduces the depth bound for exact aggregation to the feature rank $r$.

\begin{restatable}{theorem}{adaptivethreeparents}
\label{thm:b-3:exact}
For every distribution with finite second moments and feature rank $r\ge1$, the designer can choose a graph, a single-feature allocation, and an output agent achieving exact aggregation with at most three parents per agent, depth at most $r$, and at most $1+\binom r2$ agents. In particular, $\siR{b}{d}{D}=0$ for $b\ge3$ and $D\ge d$.
\end{restatable}
We sketch the construction below. The proof and further construction details are in \Cref{app:exact:feature-update}.

If $f^*=0$, one source suffices. Otherwise, choose linearly independent features $x_1,\ldots,x_r$ spanning $H$. We construct sets $J_t\subseteq[r]$ of $t$ selected features, starting with $J_0=\varnothing$. Write
\begin{equation*}
 H_t=\spn\{x_j:j\in J_t\},\qquad
 p_t=P_{H_t}f^*,\qquad
 q_{t,i}=P_{H_t+\spn\{x_i\}}f^*\quad(i\notin J_t).
\end{equation*}
We stop as soon as a selected prediction $p_t$ equals $f^*$, using its agent as the output.

After round $t$, we will have agents predicting $p_t$ and $q_{t,i}$ for every $i\notin J_t$. Consider round $t+1$. First, compare the improvements $\nrm{q_{t,i}-p_t}^2$ and select $j\notin J_t$ with the smallest nonzero improvement. Set $J_{t+1}=J_t\cup\{j\}$ and $p_{t+1}=q_{t,j}$. This choice makes the update below exact, as proved in \Cref{app:exact:feature-update}.

The agent predicting $q_{t,j}$ already predicts $p_{t+1}$ by definition. Unless we stop, what remains is constructing the agents predicting $q_{t+1,i}$ for $i\notin J_{t+1}$. We will show that an agent observing $x_i$ and receiving $p_t,p_{t+1},q_{t,i}$ predicts $q_{t+1,i}$.

Each round adds at most $r-t$ new agents and one depth, so the total number of agents is at most $1+\sum_{t=0}^{r-1}(r-t)=1+\binom r2$ and the total depth is at most $r$. The output agent predicts $p_r=f^*$.

\subsection{Depth lower bounds}
\label{subsec:b-3:depth}

We now show that the depth bounds in the two settings are optimal up to constant factors. We use a normalized version of the Gaussian example from \citet[Theorem~5.9]{kearns2026networked}. Let $Z_0,\ldots,Z_{d-1}$ be independent standard Gaussians. Define
\begin{equation}
\label{eq:b-3:ordered-features}
 x_i=\frac{Z_{i-1}-Z_i}{\sqrt2}\quad(1\le i<d),\qquad
 x_d=\frac{Z_{d-1}}{\sqrt2},\qquad
 Y=\frac{Z_0}{\sqrt2\,d}.
\end{equation}
The features sum to give $Y=d^{-1}\sum_{i=1}^d x_i$, so $f^*=Y$ and the global coefficient norm is one. Each feature has second moment at most one. Thus the distribution is normalized.

Exact aggregation for this distribution requires a path that observes $x_1,\ldots,x_d$ in order. Only $x_1$ is correlated with the label. If the incoming predictions lie in $\spn\{x_1,\ldots,x_\ell\}$ for some $\ell<d$, every feature beyond $x_{\ell+1}$ is independent of the label and those predictions. Thus $x_{\ell+1}$ is the only feature that can extend the prediction beyond this span. This refines the propagation argument used to prove \citet[Theorem~5.9]{kearns2026networked}. The proof is given in \Cref{lem:b-3:ordered-path} in \Cref{app:b-3:ordered-path}.

\begin{restatable}{theorem}{depthlowerbounds}
\label{thm:b-3:depth-lower}
In the adaptive designer setting, for every $d\ge1$, some normalized rank-$d$ distribution requires depth $D\ge d$ for exact aggregation, even without a parent limit. In the oblivious designer setting, every fixed parent limit $b\ge2$ requires $D=\Omega(d\log d)$ for exact aggregation.
\end{restatable}

\section{Two parents per agent}
\label{sec:two-parents}

We obtain the two-parent constructions by replacing each agent with three parents by a fixed gadget that uses at most two parents per agent. The gadget uses the same raw feature and reproduces the agent's prediction with only a constant increase in size and depth.

\begin{restatable}{lemma}{fixedtwoparentgadget}
\label{lem:b-2:fixed-gadget}
Consider an agent observing a raw feature $x_j$ and receiving three parent predictions $f_1,f_2,f_3$. There is a fixed gadget that reproduces this agent's prediction
\begin{equation*}
 P_{\spn\{x_j,f_1,f_2,f_3\}}Y
\end{equation*}
for every distribution with finite second moments. The gadget uses $O(1)$ agents, each observing $x_j$ and having at most two parents, and adds $O(1)$ depth above the original parents.
\end{restatable}

We sketch the proof and give the details in \Cref{app:two-parents}. We first build a replacement that may depend on the distribution, and then remove this dependence. Every agent in the gadget observes $x_j$, so the part of each prediction along $x_j$ can be split off, and we ignore $x_j$ here (\Cref{lem:b-2:three-predictions}). The goal is the three-parent agent's prediction $g=P_{\spn\{f_1,f_2,f_3\}}Y$, but each new agent can fit $Y$ from only two available predictions. Consider the map $n(f)=(\nrm{g}^2/\nrm{f}^2)f-g$ on nonzero predictions $f$. Every fit $f$ satisfies $\ip{g}{f}=\nrm{f}^2$ by \Cref{lem:least-squares-error}, so $n(f)$ is orthogonal to $g$. Thus the map sends every prediction into the two-dimensional space of vectors in $\spn\{f_1,f_2,f_3\}$ orthogonal to $g$, and $n(g)=0$. Fitting from two predictions $f_1,f_2$ gives the prediction whose image is the point closest to zero on the line through $n(f_1)$ and $n(f_2)$ (\Cref{lem:b-2:prediction-map}). So we must reach zero using only such steps.

We move to the complex plane by identifying this space with $\sC$. Multiplying all points by a complex number rotates and scales the plane about zero, so it commutes with the step. Hence if some steps turn three points $T=(T_1,T_2,T_3)$ into $\mu T=(\mu T_1, \mu T_2, \mu T_3)$, the same steps turn $\mu T$ into $\mu^2T$. Combining such steps, we can evaluate expressions in $\mu$. We construct a suitable $\mu$ and an expression in $\mu$ that equals zero. Evaluating it takes at most $300$ steps, each done by one agent, and gives zero, the image of $g$ (\Cref{lem:b-2:planar}). To remove the dependence on the distribution, the fixed gadget runs all two-parent graphs with at most $303$ agents in parallel and combines their outputs through a binary tree. Its size and depth are still constant but very large. Applying it to \Cref{thm:b-3:fixed,thm:b-3:exact} gives the same asymptotic bounds for $b=2$ in both settings (\Cref{app:b-2:bounds}).

\begin{restatable}{theorem}{twoparentbounds}
\label{thm:b-2:bounds}
With at most two parents per agent and a single-feature allocation, exact aggregation is possible for every distribution with finite second moments. In the oblivious designer setting, $O(d^2)$ agents and depth $O(d\log d)$ suffice for $d\ge2$, and one agent suffices for $d=1$. In the adaptive designer setting, $O(r^2)$ agents and depth $O(r)$ suffice for every feature rank $r\ge1$.
\end{restatable}

\section{Lower bound on the number of agents for exact aggregation}
\label{sec:agent-size}

The constructions in \Cref{sec:two-parents,sec:three-parents} use $O(d^2)$ agents. We prove a matching lower bound in the adaptive designer setting for every fixed parent limit $b\ge2$. The adversary chooses one distribution for which every exact network needs this many agents, even when the designer knows the distribution and the depth is unrestricted.

\begin{theorem}
\label{thm:size:lower-bound}
For every $d\ge2$, the adversary can choose a distribution, normalized with $M_X=A^*=1$, with $d$ linearly independent features and $Y=f^*$ such that every single-feature DAG $G=(A,E)$ achieving exact aggregation satisfies
\begin{equation}
\label{eq:size:parent-pairs}
 \sum_{i=1}^{|A|}\binom{|\Pa(A_i)|+1}{2}\ge\binom d2.
\end{equation}
\end{theorem}

If each agent has at most $b\ge1$ parents, then each summand in \eqref{eq:size:parent-pairs} is at most $\binom{b+1}{2}$, giving $|A|=\Omega(d^2/b^2)$. For fixed $b\ge2$, this proves an $\Omega(d^2)$ lower bound in the adaptive designer setting. The same bound holds in the oblivious designer setting, since a graph and allocation fixed in advance must also achieve exact aggregation on this distribution. Together with the constructions in \Cref{sec:two-parents,sec:three-parents}, this makes $\Theta(d^2)$ agents optimal in both settings.

We will first pick a matrix $E$ with certain properties and define the distribution using this matrix. Choose a symmetric $d\times d$ matrix $E$ with $|E_{ij}|<1/(8d)$ for all $i,j\in[d]$ such that the upper-triangular entries of $\Sigma=\frac12 I+E$ are algebraically independent over $\mathbb{Q}$: no nonzero polynomial with rational coefficients vanishes at these entries. Such a matrix exists by \Cref{lem:size:algebraic-independence}, proved in \Cref{app:size:algebraic-independence}. The algebraic independence will be used in \Cref{lem:size:exact-inputs}. Set $\Sigma=\frac12 I+E$ and $c=\vone_d/(4d)$, with $\vone_d$ the vector of all ones in $\R^d$. Define
\begin{equation}
\label{eq:size:distribution}
 x\sim N(0,\Sigma),\qquad
 w^*=\Sigma^{-1}c,\qquad
 Y={w^*}^\top x.
\end{equation}
By \Cref{lem:size:distribution}, the distribution is normalized with $M_X=A^*=1$. The definitions also give $\E[xx^\top]=\Sigma$ and $\E[xY]=\Sigma w^*=c$. Fix a graph $G=(A,E)$ and a single-feature allocation on the chosen distribution. Each fitted prediction has a unique representation $f_i=w_i^\top x$ with $w_i\in\R^d$, since the features are linearly independent by \Cref{lem:size:distribution}.
By \eqref{eq:size:fixed-correlations}, the squared error of $f_i$ is $\E[Y^2]+w_i^\top\Sigma w_i-2w_i^\top c$. Let $e_1,\ldots,e_d$ be the coordinate vectors. Agent $A_i$ observing $x_\ell$ fits a linear combination of this feature and its parents' predictions. Omitting the constant $\E[Y^2]$ therefore gives
\begin{equation}
\label{eq:size:agent-minimization}
 w_i=\argmin_{w\in\spn(\{e_\ell\}\cup\{w_j:A_j\in\Pa(A_i)\})}
 \left(w^\top\Sigma w-2w^\top c\right).
\end{equation}

We seek a nonzero symmetric matrix $\Delta$ such that replacing $\Sigma$ by $\Sigma+t\Delta$, with $c$ fixed, preserves every fitted coefficient vector. For sufficiently small $|t|$, the new matrix remains positive definite and defines a distribution as in \eqref{eq:size:distribution}. For fixed parent coefficient vectors, the objective in \eqref{eq:size:agent-minimization} changes by $t w^\top\Delta w$. Requiring $u^\top\Delta v=0$ for every pair of input coefficient vectors at each agent, including $u=v$, makes this change zero throughout each allowed span by bilinearity. Induction along the graph then preserves all fitted coefficients. The following lemma constructs such a $\Delta$ under the stated count bound, while \Cref{lem:size:exact-inputs} rules it out for an exact network on our chosen distribution.

\begin{restatable}{lemma}{sizeinputcount}
\label{lem:size:input-count}
Fix a single-feature DAG on a distribution with $d$ linearly independent features. If
\begin{math}
 \sum_{i=1}^{|A|}\binom{|\Pa(A_i)|+1}{2}<\binom d2,
\end{math}
there is a nonzero symmetric matrix $\Delta$ with zero diagonal such that
\begin{equation}
\label{eq:size:input-direction}
 u^\top\Delta v=0
\end{equation}
for every agent $A_i$ observing $x_\ell$ and all $u,v\in\{e_\ell\}\cup\{w_j:A_j\in\Pa(A_i)\}$, including $u=v$.
\end{restatable}

The proof, in \Cref{app:size:input-count}, counts equations. The matrix $\Delta$ has $\binom d2$ unknown entries, and an agent with $p$ parents imposes at most $p+\binom p2=\binom{p+1}{2}$ homogeneous linear equations on them. The next lemma, proved in \Cref{app:size:exact-inputs}, rules out a nonzero $\Delta$ for an exact network on our distribution. Together with \Cref{lem:size:input-count}, it gives \Cref{thm:size:lower-bound} (\Cref{app:size:lower-bound}).

\begin{restatable}{lemma}{sizeexactinputs}
\label{lem:size:exact-inputs}
For the distribution fixed in \Cref{eq:size:distribution}, suppose a single-feature DAG achieves exact aggregation. If a symmetric matrix $\Delta$ satisfies \eqref{eq:size:input-direction} at every agent, then $\Delta=0$.
\end{restatable}

\subsection*{AI use statement}

We used OpenAI's GPT-6 Astra in Pro mode to help develop critical ingredients for proving mathematical claims, including key proof ideas, and to assist with proof writing. Specifically, it assisted with the construction and proof of the gadget that replaces a three-parent agent by two-parent agents in \Cref{sec:two-parents,app:two-parents}. It also assisted with the proof of the lower bound on the number of agents in \Cref{sec:agent-size,app:agent-size}. We did not use generative AI tools to develop theoretical models or conceptual frameworks, formulate mathematical claims, propose or refine hypotheses, design or assess research methods or experiments, or interpret results. Generating synthetic datasets, implementing methods, translation, cleaning or reformatting datasets, and qualitative or thematic data analysis are not applicable to this work. We verified the correctness of all AI-assisted proofs. We take responsibility for the final content of this paper, including all AI-assisted work.

\printbibliography

\clearpage
\appendix
\section{Extended Related Work}
\label{app:extended-related-work}

\paragraph{Prediction exchange and calibration.}
Two parties with different features can learn from one another by taking turns making and revising predictions. \citet{collina2026collaborative} give protocols for this task whose predictions compete with a restricted class of policies on the parties' joint feature space. The parties never share their raw features. Instead, the protocols call learning algorithms on each party's own feature space, with guarantees for both online prediction and learning from a fixed distribution. This approach builds on the agreement protocols of \citet{collina2025tractable}, which use calibration conditions to relax the assumptions of Bayesian agreement. In the classical result of \citet{aumann1976agreeing}, two agents with a common prior must agree on the probability of an event once their posterior probabilities are common knowledge. The calibration conditions support efficient prediction exchange without requiring the parties to know a common prior.

An auditor for multiaccuracy looks for a test function that correlates with a predictor's errors. Finding one gives a direction in which to correct the predictor. \citet{kim2019multiaccuracy} use this idea to post-process a given predictor until the expected product of its residual with each test is small. An indicator test measures the mean residual within a group, weighted by the group's probability. The normal equations for least squares give zero inner product between the residual and each input, whether that input is a raw feature or another learner's prediction.

\paragraph{Social learning and opinion dynamics.}
Repeated averaging is the update rule in \citet{degroot1974reaching}: each agent takes a weighted average of its neighbors' current beliefs, with the weights fixed throughout the process. When the agents reach consensus, the common belief is a weighted average of their initial beliefs. \citet{golub2010naive} study whether this consensus approaches the true state as the network grows, assuming independent noisy initial estimates. They characterize learning through the weights in the final consensus. Convergence to the true state holds precisely when the largest weight on any individual's initial estimate tends to zero.

\paragraph{Stacking and distributed learning.}
In stacked generalization, a second learner is trained to combine predictions from other models \citep{wolpert1992stacked}. Its training inputs are predictions on examples held out when fitting those models. This gives the second learner evidence about their errors on unseen data. Vertical federated learning organizes collaboration around training a shared model from feature columns held by different parties \citep{yang2019federated}. The parties have records for common examples but keep their raw data local. For example, SecureBoost trains boosted trees by exchanging encrypted gradient statistics \citep{cheng2021secureboost}. These statistics allow the parties to evaluate candidate splits using features held at different sites, and the exchanges continue as further splits and trees are added.

\section{Proofs for the preliminaries}
\label{app:preliminaries}

We restate and prove the lemmas from \Cref{sec:preliminaries}.

\leastsquareserror*
\begin{proof}
The residual $Y-f$ is orthogonal to $V$. Taking its inner product with $f\in V$ gives $\ip{Y}{f}=\ip{Y-f}{f}+\ip{f}{f}=\nrm{f}^2$. Since $f-g\in V$, the two terms in $Y-g=(Y-f)+(f-g)$ are orthogonal. The Pythagorean identity gives \eqref{eq:excess-error}.
\end{proof}

\replaceyfstar*
\begin{proof}
Let $f_i$ and $f'_i$ be the fitted predictions for the labels $Y$ and $f^*$, respectively. We will show, by induction on a topological order of the agents, that $f_i=f'_i$ for every agent $A_i$. For any agent, assume that its incoming predictions are unchanged for the two labels. Then its input space $V_i$ is unchanged for the two labels. Since $Y-f^*$ is orthogonal to $V_i$ and projection is linear, we have $f_i=P_{V_i}Y=P_{V_i}(Y-f^*)+P_{V_i}f^*=P_{V_i}f^*=f'_i$.
\end{proof}

\minimaxrelations*
\begin{proof}
Excess error is nonnegative because $f^*$ minimizes MSE over $H$. Every agent can use the zero predictor, so its fitted MSE is at most $\MSE(0)$. By \Cref{lem:least-squares-error}, its excess error is therefore at most $\nrm{f^*}^2$. The unit bounds give $\nrm{f^*}\le\sum_{i=1}^d|w_i^*|\nrm{x_i}\le1$.

For any fixed graph, allocation, and output agent, the worst error over distributions is at least the best achievable error on each distribution. Taking the supremum of the latter over distributions gives $\siR{b}{d}{D}$. Taking the infimum of the former over graphs, allocations, and output agents then proves $\siR{b}{d}{D}\le\siBR{b}{d}{D}$.

Increasing $b$ or $D$ allows more choices in each infimum and thus cannot increase either quantity. Increasing $d$ cannot decrease either quantity, since the adversary can set $x_{d+1}=x_1$ to recover the $d$-feature problem.
\end{proof}

\section{Proofs for one parent per agent}
\label{app:one-parent}

\subsection{One or two features}
\label{app:b-1:two-features}

\begin{proposition}
\label{prop:b-1:two-features}
For every distribution with finite second moments and $d\le2$ features, the designer can choose a feature allocation on a path of $d$ agents that achieves exact aggregation. For $d=2$, the fixed allocation $x_1,x_2,x_1$ on a path of three agents achieves exact aggregation for every such distribution.
\end{proposition}
\begin{proof}
For $d=1$, a single agent sees the full feature space and predicts $f^*$. For $d=2$, if both $\ip{Y}{x_1}$ and $\ip{Y}{x_2}$ are zero, then $f^*=0$ and every agent predicts zero by \Cref{lem:replace-y-fstar}. Otherwise, choose $i\in\{1,2\}$ with $\ip{Y}{x_i}\ne0$ and give $x_i$ to the first agent. Its prediction is a nonzero multiple of $x_i$. Give the other feature to the second agent. Its inputs span $H$, so it predicts $f^*$.

Now fix the allocation $x_1,x_2,x_1$. If $\ip{Y}{x_1}\ne0$, the first prediction is a nonzero multiple of $x_1$ and the second agent predicts $f^*$. The third agent also predicts $f^*$, because its inputs include $f^*$ and lie in $H$. If $\ip{Y}{x_1}=0$ but $\ip{Y}{x_2}\ne0$, the first prediction is zero and the second is a nonzero multiple of $x_2$. The third agent then has inputs spanning $H$ and predicts $f^*$. If both inner products vanish, every prediction and $f^*$ are zero.
\end{proof}

\subsection{The lower bound}
\label{app:b-1:lower-bound}

Use the distribution in \eqref{eq:b-1:lower-bound-distribution}, with $\rho^2=1/(40D)$. Together, the features determine the label: $f^*=(2x_1-x_2-x_3)/3=Y$. The sum of the absolute values of these coefficients is $4/3$. The feature second moments are $1+\rho^2$, $1+\rho^2$, and $1+2\rho^2$, all at most $2$. Thus the distribution satisfies constant bounds on the feature moments and global coefficients. Since $f^*=Y$, an agent's MSE is also its excess error.

We first prove the two estimates used to bound error along the path, then prove \Cref{thm:b-1:lower-bound}.

\begin{lemma}
\label{lem:b-1:lower-bound-two-features}
For the distribution in \eqref{eq:b-1:lower-bound-distribution}, let $f$ be the best linear predictor from any two distinct features. Then $\MSE(f)\ge\rho^2/6$ and $\nrm{f}^2\ge\rho^2/5$.
\end{lemma}
\begin{proof}
We compute the error for each pair by projecting $Y$ onto the direction orthogonal to its span. We then use the error to bound the norm of the fit.

The variables $U,V,Z$ are orthonormal in $L^2(\gD)$ because they are independent standard Gaussians. Each pair of features is linearly independent: both features have $Z$ coefficient one, so they could be proportional only if they were equal, but their $U,V$ coefficients differ. Thus the orthogonal complement of each pair's span within $\spn\{U,V,Z\}$ is a line. Since $Y$ also lies in this three-dimensional space, the residual $Y-f$ is its projection onto that line. For a nonzero vector $\xi$ on the line, this gives
\begin{equation*}
 Y-f=\frac{\ip{Y}{\xi}}{\nrm{\xi}^2}\xi,
 \qquad
 \MSE(f)=\frac{\ip{Y}{\xi}^2}{\nrm{\xi}^2}.
\end{equation*}

For the pair $\{x_1,x_2\}$, take $\xi=U+V-\rho Z$. Using $x_1=Z+\rho U$ and $x_2=Z+\rho V$, we have
\begin{equation*}
 \ip{\xi}{x_1}=\rho-\rho=0,
 \qquad
 \ip{\xi}{x_2}=\rho-\rho=0.
\end{equation*}
Since $Y=\rho U$, we also have $\ip{Y}{\xi}=\rho$ and $\nrm{\xi}^2=1+1+\rho^2=2+\rho^2$. The error for this pair is therefore
\begin{equation*}
 \MSE(f)=\frac{\rho^2}{2+\rho^2}.
\end{equation*}

For the pair $\{x_1,x_3\}$, take $\xi=U-2V-\rho Z$. Using $x_3=Z-\rho U-\rho V$, we get
\begin{equation*}
 \ip{\xi}{x_1}=\rho-\rho=0,
 \qquad
 \ip{\xi}{x_3}=-\rho+2\rho-\rho=0.
\end{equation*}
Here $\ip{Y}{\xi}=\rho$ and $\nrm{\xi}^2=1+4+\rho^2=5+\rho^2$, so the error is
\begin{equation*}
 \MSE(f)=\frac{\rho^2}{5+\rho^2}.
\end{equation*}

For the pair $\{x_2,x_3\}$, take $\xi=2U-V+\rho Z$. In this case,
\begin{equation*}
 \ip{\xi}{x_2}=-\rho+\rho=0,
 \qquad
 \ip{\xi}{x_3}=-2\rho+\rho+\rho=0.
\end{equation*}
Now $\ip{Y}{\xi}=2\rho$ and $\nrm{\xi}^2=4+1+\rho^2=5+\rho^2$, giving
\begin{equation*}
 \MSE(f)=\frac{(2\rho)^2}{5+\rho^2}=\frac{4\rho^2}{5+\rho^2}.
\end{equation*}

Since $\rho^2\le1/40$, each denominator is at most $6$ and each numerator is at least $\rho^2$, so every pair has error at least $\rho^2/6$. The three errors are at most $\rho^2/2$, $\rho^2/5$, and $4\rho^2/5$, respectively, so every pair also has error at most $4\rho^2/5$.

Finally, the fit $f$ and its residual $Y-f$ are orthogonal, so $\nrm{Y}^2=\nrm{f}^2+\MSE(f)$. The upper bound on the error therefore gives
\begin{equation*}
 \nrm{f}^2=\nrm{Y}^2-\MSE(f)
 \ge\rho^2-\frac{4\rho^2}{5}
 =\frac{\rho^2}{5}. \qedhere
\end{equation*}
\end{proof}

\begin{lemma}
\label{lem:b-1:lower-bound-step}
For consecutive agents $A_t,A_{t+1}$ on a single-feature path for \eqref{eq:b-1:lower-bound-distribution}, if $\nrm{f_t}^2\ge\rho^2/5$, then
\begin{equation*}
 \MSE(f_{t+1})\ge(1-10\rho^2)\MSE(f_t).
\end{equation*}
\end{lemma}
\begin{proof}
We show that the next agent can remove at most a $10\rho^2$ fraction of its parent's error. Let $x_i$ and $x_j$ be the features observed by $A_t$ and $A_{t+1}$, respectively, and write $e=Y-f_t$ for the parent's residual. Since $f_t$ is the fit from the parent's inputs, $e$ is orthogonal to their span. In particular, $e\perp x_i$ and $e\perp f_t$.

The next agent receives $f_t$ and observes $x_j$, so it can improve on $f_t$ only through the part of $x_j$ orthogonal to $f_t$. The hypothesis $\nrm{f_t}\ge\rho/\sqrt5>0$ lets us define this part as
\begin{equation*}
 z=x_j-\frac{\ip{x_j}{f_t}}{\nrm{f_t}^2}f_t.
\end{equation*}
Thus $z\perp f_t$ and $\spn\{f_t,x_j\}=\spn\{f_t,z\}$. If $z=0$, the next agent's input span is just $\spn\{f_t\}$. Since $e\perp f_t$, its fit remains $f_t$, which proves the claim in this case.

Suppose now that $z\ne0$. The projection of $Y$ onto the line spanned by $f_t$ is $f_t$, because $Y=f_t+e$ and $e\perp f_t$. Its projection onto the orthogonal line spanned by $z$ is $\ip{e}{z}z/\nrm{z}^2$. The next agent therefore predicts
\begin{equation*}
 f_{t+1}=f_t+\frac{\ip{e}{z}}{\nrm{z}^2}z.
\end{equation*}
Since $f_t$ belongs to the next agent's input span, \Cref{lem:least-squares-error} gives the exact decrease in error:
\begin{equation*}
 \MSE(f_t)-\MSE(f_{t+1})
 =\nrm{f_{t+1}-f_t}^2
 =\frac{|\ip{e}{z}|^2}{\nrm{z}^2}.
\end{equation*}
It remains to bound the numerator by $5\rho^2\nrm{e}^2$ and the denominator from below by $1/2$.

For the numerator, we use the fact that the features are close together. Subtracting any two features in \eqref{eq:b-1:lower-bound-distribution} cancels their common term $Z$. Since $U,V$ are orthonormal, the three squared distances are
\begin{align*}
 \nrm{x_1-x_2}^2&=\rho^2\nrm{U-V}^2=2\rho^2,\\
 \nrm{x_1-x_3}^2&=\rho^2\nrm{2U+V}^2=5\rho^2,\\
 \nrm{x_2-x_3}^2&=\rho^2\nrm{U+2V}^2=5\rho^2.
\end{align*}
Thus $\nrm{x_j-x_i}^2\le5\rho^2$, also when $i=j$. Since $z$ differs from $x_j$ by a multiple of $f_t$ and $e\perp f_t$, we have $\ip{e}{z}=\ip{e}{x_j}$. Using $e\perp x_i$ and then Cauchy--Schwarz gives
\begin{align*}
 |\ip{e}{z}|^2
 &=|\ip{e}{x_j-x_i}|^2\\
 &\le\nrm{e}^2\nrm{x_j-x_i}^2
 \le5\rho^2\nrm{e}^2.
\end{align*}

For the denominator, we show that removing the component along $f_t$ leaves most of the squared norm of $x_j$. Since $Y=\rho U$ and $U,V,Z$ are orthonormal,
\begin{equation*}
 \ip{Y}{x_1}=\rho^2,\qquad
 \ip{Y}{x_2}=0,\qquad
 \ip{Y}{x_3}=-\rho^2.
\end{equation*}
Thus $|\ip{Y}{x_i}|\le\rho^2$ for every feature. Since $e=Y-f_t$ is orthogonal to $x_i$, this also gives $|\ip{f_t}{x_i}|=|\ip{Y}{x_i}|\le\rho^2$.

Writing $x_j=x_i+(x_j-x_i)$, the triangle inequality followed by Cauchy--Schwarz gives
\begin{align*}
 |\ip{f_t}{x_j}|
 &\le |\ip{f_t}{x_i}|+|\ip{f_t}{x_j-x_i}|\\
 &\le |\ip{f_t}{x_i}|+\nrm{f_t}\nrm{x_j-x_i}.
\end{align*}
Dividing by $\nrm{f_t}$ and using the hypothesis $\nrm{f_t}\ge\rho/\sqrt5$, together with the bounds just proved, we obtain
\begin{align*}
 \frac{|\ip{f_t}{x_j}|}{\nrm{f_t}}
 &\le\frac{|\ip{f_t}{x_i}|}{\nrm{f_t}}+\nrm{x_j-x_i}\\
 &\le\frac{\rho^2}{\rho/\sqrt5}+\sqrt5\rho
 =2\sqrt5\rho.
\end{align*}
The left-hand side is the norm of the component removed from $x_j$ to obtain $z$. The three features have squared norms $1+\rho^2$, $1+\rho^2$, and $1+2\rho^2$, so $\nrm{x_j}^2\ge1$. Since the removed component is orthogonal to $z$, Pythagoras gives
\begin{equation*}
 \nrm{z}^2
 =\nrm{x_j}^2-\frac{|\ip{x_j}{f_t}|^2}{\nrm{f_t}^2}
 \ge1-20\rho^2
 \ge\frac12.
\end{equation*}
The last inequality uses $\rho^2=1/(40D)\le1/40$.

Substituting the numerator and denominator bounds into the error decrease formula, and using $\nrm{e}^2=\MSE(f_t)$, we conclude that
\begin{equation*}
 \MSE(f_t)-\MSE(f_{t+1})
 \le\frac{5\rho^2\nrm{e}^2}{1/2}
 =10\rho^2\MSE(f_t).
\end{equation*}
Rearranging proves the claim.
\end{proof}

\begin{theorem}
\label{thm:b-1:lower-bound}
For every integer $D\ge1$, the distribution in \eqref{eq:b-1:lower-bound-distribution} satisfies
\begin{equation*}
 \MSE(f_G)-\MSE(f^*)\ge\frac1{320D}
\end{equation*}
for every graph $G$ with $\Delta^-(G)\le1$ and $\depth(A_G)\le D$, every single-feature allocation, and every output agent $A_G$.
\end{theorem}
\begin{proof}
Since every agent has at most one parent, tracing parents backward from the output gives a single path. All other agents can be removed because their predictions do not affect the output. Write the remaining agents as $A_1,\ldots,A_n$, with $A_n=A_G$. The depth bound gives $n\le D$. Since $f^*=Y$, it suffices to show that $\MSE(f_n)\ge1/(320D)$.

We first handle paths that never combine a nonzero prediction with a different raw feature. For the other paths, \Cref{lem:b-1:lower-bound-two-features} will give an initial error bound, and \Cref{lem:b-1:lower-bound-step} will bound the improvement at each later agent.

An agent observing $x_2$ with no parent or a zero parent prediction has input span $\spn\{x_2\}$. Its fit is zero because $\ip{Y}{x_2}=0$. By induction along the path, all predictions are therefore zero until an agent observes $x_1$ or $x_3$. If neither feature appears, the output error is $\MSE(f_n)=\nrm{Y}^2=\rho^2=1/(40D)\ge1/(320D)$.

Otherwise, let $A_s$ be the first agent observing a feature $x_i\in\{x_1,x_3\}$. Its parent, if present, predicts zero, so
\begin{equation*}
 f_s=\frac{\ip{Y}{x_i}}{\nrm{x_i}^2}x_i\ne0.
\end{equation*}
The prediction is nonzero because $\ip{Y}{x_1}=\rho^2$ and $\ip{Y}{x_3}=-\rho^2$. If the next agent also observes $x_i$, its inputs still span $\spn\{x_i\}$, so its fit is again $f_s$. This remains true for as long as the path repeats $x_i$.

If every agent after $A_s$ observes $x_i$, then $f_n=f_s$. The formula for $f_s$ gives
\begin{equation*}
 \nrm{f_s}^2
 =\frac{|\ip{Y}{x_i}|^2}{\nrm{x_i}^2}
 =\frac{\rho^4}{\nrm{x_i}^2}
 \le\rho^4,
\end{equation*}
since $\nrm{x_1}^2=1+\rho^2$ and $\nrm{x_3}^2=1+2\rho^2$ are at least one. The fit and its residual are orthogonal, so
\begin{equation*}
 \MSE(f_n)=\nrm{Y}^2-\nrm{f_s}^2
 \ge\rho^2-\rho^4
 \ge\frac{\rho^2}{2}
 =\frac1{80D}
 \ge\frac1{320D}.
\end{equation*}
Here we used $\rho^2=1/(40D)\le1/40<1/2$.

It remains to consider paths that use a different feature after $A_s$. Let $A_\tau$ be the first such agent, and call its feature $x_j$. Its parent still predicts $f_s$, a nonzero multiple of $x_i$, so
\begin{equation*}
 \spn\{f_{\tau-1},x_j\}=\spn\{x_i,x_j\}.
\end{equation*}
Thus $f_\tau$ is the fit from two distinct raw features. By \Cref{lem:b-1:lower-bound-two-features},
\begin{equation*}
 \MSE(f_\tau)\ge\frac{\rho^2}{6},\qquad
 \nrm{f_\tau}^2\ge\frac{\rho^2}{5}.
\end{equation*}

To apply \Cref{lem:b-1:lower-bound-step} at every later step, we must check that the squared norm stays at least $\rho^2/5$. Each agent can use its parent's prediction, so its fit has no larger error. Also, \Cref{lem:least-squares-error} gives $\nrm{f_t}^2=\nrm{Y}^2-\MSE(f_t)$ for every agent. Consequently, for every $\tau\le t\le n$,
\begin{equation*}
 \nrm{f_t}^2
 =\rho^2-\MSE(f_t)
 \ge\rho^2-\MSE(f_\tau)
 =\nrm{f_\tau}^2
 \ge\frac{\rho^2}{5}.
\end{equation*}
The lemma therefore applies at each of the $n-\tau$ remaining steps. Since its factor $1-10\rho^2=1-1/(4D)$ is positive, iterating gives
\begin{equation*}
 \MSE(f_n)
 \ge\MSE(f_\tau)(1-10\rho^2)^{n-\tau}
 \ge\frac{\rho^2}{6}\left(1-\frac1{4D}\right)^{n-\tau}.
\end{equation*}
Finally, Bernoulli's inequality bounds the power below by $1-(n-\tau)/(4D)$. Since $n-\tau\le D$, at least three quarters of the error at $A_\tau$ remains. Hence
\begin{align*}
 \MSE(f_n)
 &\ge\frac{\rho^2}{6}\left(1-\frac{n-\tau}{4D}\right)\\
 &\ge\frac{\rho^2}{6}\cdot\frac34
 =\frac{\rho^2}{8}
 =\frac1{320D}.
\end{align*}
This proves the bound for every path and therefore for every graph in the theorem.
\end{proof}

\subsection{The upper bound}
\label{app:b-1:upper-bound}

We prove the upper bound for general normalization constants using the greedy allocation in \eqref{eq:b-1:greedy-allocation}.

\begin{theorem}
\label{thm:b-1:upper-bound}
Let $\gD$ satisfy \Cref{def:normalized-distribution} at bounds $M_X,A^*>0$. For every integer $D\ge1$, the designer can choose a feature allocation on the path $A_1,\ldots,A_D$ such that
\begin{equation*}
 \MSE(f_D)-\MSE(f^*)\le\frac{(A^*M_X)^2}{D+1}.
\end{equation*}
\end{theorem}

We first bound the decrease in error from one agent to the next, then prove \Cref{thm:b-1:upper-bound}.

\begin{lemma}
\label{lem:b-1:upper-bound-step}
Under the assumptions of \Cref{thm:b-1:upper-bound}, suppose at least one feature has positive norm. The allocation in \eqref{eq:b-1:greedy-allocation}, starting from $f_0=0$, satisfies, for every $0\le t<D$,
\begin{equation}
\label{eq:b-1:upper-bound-recurrence}
 \nrm{f^*-f_{t+1}}^2
 \le\nrm{f^*-f_t}^2
 -\frac{\nrm{f^*-f_t}^4}{(A^*M_X)^2}.
\end{equation}
\end{lemma}
\begin{proof}
Fix $0\le t<D$, and let $x_i$ be the feature chosen by \eqref{eq:b-1:greedy-allocation}. By \Cref{lem:replace-y-fstar}, the prediction $f_t$ is the projection of $f^*$ onto $V_t$, so $f^*-f_t$ is orthogonal to $f_t$. This also holds for $f_0=0$. Since $Y-f^*$ is orthogonal to every raw feature,
\begin{align*}
 \nrm{f^*-f_t}^2
 &=\ip{f^*-f_t}{f^*}
 =\sum_{\ell=1}^d w_\ell^*\ip{Y-f_t}{x_\ell}\\
 &\le\left(\sum_{\ell=1}^d|w_\ell^*|\nrm{x_\ell}\right)
 \frac{|\ip{Y-f_t}{x_i}|}{\nrm{x_i}}\\
 &\le A^*M_X\frac{|\ip{Y-f_t}{x_i}|}{\nrm{x_i}}.
\end{align*}
The first inequality uses the maximizing choice of $i$ in \eqref{eq:b-1:greedy-allocation}. The last inequality uses the coefficient and feature bounds.

The next agent can use $f_t+\alpha x_i$ for any $\alpha\in\R$. Its MSE is
\begin{equation*}
 \MSE(f_t+\alpha x_i)
 =\MSE(f_t)-2\alpha\ip{Y-f_t}{x_i}+\alpha^2\nrm{x_i}^2.
\end{equation*}
Choosing $\alpha=\ip{Y-f_t}{x_i}/\nrm{x_i}^2$ decreases MSE by $|\ip{Y-f_t}{x_i}|^2/\nrm{x_i}^2$. The fitted prediction has no larger error, so
\begin{equation*}
 \MSE(f_t)-\MSE(f_{t+1})
 \ge\frac{|\ip{Y-f_t}{x_i}|^2}{\nrm{x_i}^2}
 \ge\frac{\nrm{f^*-f_t}^4}{(A^*M_X)^2}.
\end{equation*}
By \Cref{lem:least-squares-error}, the left-hand side equals $\nrm{f^*-f_t}^2-\nrm{f^*-f_{t+1}}^2$, including when $t=0$. Rearranging proves \eqref{eq:b-1:upper-bound-recurrence}.
\end{proof}

\begin{proof}[Proof of \Cref{thm:b-1:upper-bound}]
If all features have norm zero, every agent is exact. Otherwise, use the allocation in \eqref{eq:b-1:greedy-allocation} and let $e_t=\nrm{f^*-f_t}^2$ for $0\le t\le D$. The initial error satisfies $e_0=\nrm{f^*}^2\le(A^*M_X)^2$ by the coefficient and feature bounds. If the error reaches zero, it stays zero because each later agent can use its parent's prediction.

For positive errors, \Cref{lem:b-1:upper-bound-step} gives $e_t-e_{t+1}\ge e_t^2/(A^*M_X)^2$. Dividing by $e_te_{t+1}$ and using $e_{t+1}\le e_t$ shows that $1/e_{t+1}-1/e_t\ge1/(A^*M_X)^2$. Summing over the $D$ agents gives
\begin{equation*}
 \frac1{e_D}\ge\frac1{e_0}+\frac{D}{(A^*M_X)^2}
 \ge\frac{D+1}{(A^*M_X)^2}.
\end{equation*}
Taking reciprocals gives the desired bound, since $e_D$ is the output's excess error by \Cref{lem:least-squares-error}.
\end{proof}

\subsection{The bounds for the unit case}
\label{app:b-1:rate}

\pathrate*
\begin{proof}
The upper bound follows from \Cref{thm:b-1:upper-bound} with $M_X=A^*=1$.

For the lower bound, multiply the features in \eqref{eq:b-1:lower-bound-distribution} by $2\sqrt2/3$ and the label by $1/\sqrt2$. The new global coefficients are $1/2,-1/4,-1/4$, whose absolute values sum to one. The feature second moments are at most $(8/9)(1+2\rho^2)\le14/15<1$, so the rescaled distribution belongs to $\gC_3$.

Rescaling the features by a nonzero constant preserves their spans. Rescaling the label by $1/\sqrt2$ divides every fitted prediction by $\sqrt2$, by linearity of projection and induction along the path. Every excess error is therefore halved. Thus \Cref{thm:b-1:lower-bound} gives $\siR{1}{3}{D}\ge1/(640D)$. Monotonicity in $d$ from \Cref{prop:minimax} extends this bound to all $d\ge3$. Finally, $\siBR{1}{d}{D}\ge\siR{1}{d}{D}$ by the same proposition.
\end{proof}

\subsection{A fixed allocation at depth below the number of features}
\label{app:b-1:fixed-path}

A path fixed in advance may omit a feature on which the label depends entirely. This prevents a uniform upper bound below one when the path is shorter than the number of features.

\begin{proposition}
\label{prop:b-1:fixed-path}
For all positive integers $d,D$ with $D<d$, we have $\siBR{1}{d}{D}=1$.
\end{proposition}
\begin{proof}
Fix a graph, allocation, and output agent with at most one parent per agent and output depth at most $D$. The path ending at the output contains at most $D<d$ agents, so some feature $x_j$ is not observed on that path. Choose independent standard Gaussian features and let $Y=x_j$. This distribution belongs to $\gC_d$, and the global predictor is $f^*=Y$.

Each feature observed on the path is orthogonal to $Y$. Starting at the source, induction shows that every prediction on the path is zero. The output therefore has excess error $\nrm{Y}^2=1$. This proves the lower bound for every fixed graph, allocation, and output agent. The upper bound of one follows from \Cref{prop:minimax}.
\end{proof}

\section{Proofs for three parents per agent}
\label{app:three-parents}

\subsection{A fixed graph for exact aggregation}
\label{app:b-3:fixed}

We first construct the gadget in \Cref{lem:b-3:history-extension}, which computes the fit from the earlier predictions together with one incoming prediction. We then use this gadget to prove \Cref{thm:b-3:fixed}. We use $f^*$ as the label throughout, as permitted by \Cref{lem:replace-y-fstar}.

Fix $p_0,\ldots,p_t,q$ satisfying the hypotheses of \Cref{lem:b-3:history-extension}. The gadget must compute the fit from $x_1,p_0,\ldots,p_t,q$. We will do this by computing fits from shorter intervals of the list $p_0,\ldots,p_t$ and then joining them. For $0\le\ell\le r\le t$, define
\begin{equation*}
 W_{\ell,r}=\spn\{x_1,p_t,p_\ell,\ldots,p_r\},\qquad
 g_{\ell,r}=P_{W_{\ell,r}+\spn\{q\}}f^*.
\end{equation*}
Thus $g_{\ell,r}$ is the fit from the predictions in the interval $[\ell,r]$, together with $x_1,p_t,q$. Since $W_{0,t}=V_t$, the gadget's output must be $g_{0,t}$. Every interval includes $p_t$, which is the fit from all of $V_t$. In particular, $P_{W_{\ell,r}}f^*=p_t$: the vector $p_t$ belongs to $W_{\ell,r}$, and its residual $f^*-p_t$ is orthogonal to $W_{\ell,r}\subseteq V_t$.

For $i\le j$, we have $p_i\in V_i\subseteq V_j$, while $f^*-p_j\perp V_j$. This gives the first equality below, and \Cref{lem:least-squares-error} applied to $p_i=P_{V_i}f^*$ gives the second:
\begin{equation}
\label{eq:b-3:nested-predictions}
 \ip{p_i}{p_j}=\ip{p_i}{f^*}=\nrm{p_i}^2,
 \qquad
 \nrm{p_j-p_i}^2=\nrm{p_j}^2-\nrm{p_i}^2.
\end{equation}
The squared-distance identity follows by expanding $\nrm{p_j-p_i}^2$ and substituting the inner-product identity.

We begin with intervals containing two consecutive predictions. The next lemma computes their fit using at most two agents, each with at most three parents.

\begin{lemma}
\label{lem:b-3:history-leaf}
Let $0\le\ell<r\le t$ with $r=\ell+1$. The prediction $g_{\ell,r}$ can be computed by the following agents, all observing $x_1$. If $\ell=0$, one agent with parents $p_t,p_r,q$ suffices. If $r=t$, one agent with parents $p_t,p_\ell,q$ suffices. For $0<\ell<r<t$, first create an agent with parents $p_\ell,p_r,q$ and call its prediction $u$. An agent with parents $p_r,p_t,u$ then predicts $g_{\ell,r}$.
\end{lemma}
\begin{proof}
We first handle the two end intervals. If $\ell=0$, then $p_\ell=p_0\in\spn\{x_1\}$, so
\begin{equation*}
 W_{0,r}+\spn\{q\}=\spn\{x_1,p_t,p_r,q\}.
\end{equation*}
An agent observing $x_1$ with parents $p_t,p_r,q$ therefore predicts $g_{0,r}$. If $r=t$, the corresponding space is $\spn\{x_1,p_t,p_\ell,q\}$, which the other one-agent construction uses directly.

Now suppose $0<\ell<r<t$. The first agent observes $x_1$ and receives $p_\ell,p_r,q$. Write $K=\spn\{x_1,p_\ell,p_r\}$ for its input space before adding $q$. The best prediction from $K$, namely $P_K f^*$, is $p_r$: it is already the best prediction in the larger space $V_r$, and it belongs to $K$. After adding $q$, the first agent predicts $u=P_{K+\spn\{q\}}f^*$. We will show why the second agent can use $p_r,p_t,u$ in place of the predictions spanning its required space $p_\ell,p_r,p_t,q$ for $g_{\ell,r}$.

First consider the case $u=p_r$. Since the first agent receives $q$, its prediction $p_r$ has no greater error than $q$. The agent producing $q$ receives $p_t$ by hypothesis, so $q$ has no greater error than $p_t$. Finally, $p_t$ has no greater error than $p_r$ because it is the best prediction in $V_t$ and $p_r\in V_r\subseteq V_t$. These comparisons give
\begin{equation*}
 \nrm{f^*-p_r}^2
 \le\nrm{f^*-q}^2
 \le\nrm{f^*-p_t}^2
 \le\nrm{f^*-p_r}^2.
\end{equation*}
Hence all three errors are equal, and by \Cref{lem:least-squares-error}, $u=p_r=q=p_t$. Thus all three parents of the second agent predict $p_r$, and $g_{\ell,r}=P_Kf^*=p_r$. The second agent retains this fit because its inputs lie in $K$ and include $p_r$.

Now suppose $u\ne p_r$. Then we must have $u \notin K$ because $p_r$ is the best prediction in $K$. The first agent's input space extends $K$ by at most one dimension by adding $q$, and its prediction $u$ must use that dimension because otherwise $u \in K$. Consequently,
\begin{equation*}
 K+\spn\{u\}=K+\spn\{q\}.
\end{equation*}
Replacing $q$ with $u$ therefore leaves the full input space for $g_{\ell,r}$ unchanged:
\begin{equation*}
 g_{\ell,r}=P_{K+\spn\{p_t,q\}}f^*
 =P_{K+\spn\{p_t,u\}}f^*.
\end{equation*}

It remains to explain why
\begin{equation*}
  P_{\spn\{x_1,p_\ell,p_r,p_t,u\}}f^* = P_{\spn\{x_1,p_r,p_t,u\}}f^*,
\end{equation*}
where the left-hand side is the projection onto $K+\spn\{p_t,u\}$ and the right-hand side is the prediction of the second agent. So it remains to show that the extra $p_\ell$ in the left-hand side does not change the projection.

We will show that the component of $g_{\ell,r}$ in $K$ is still $p_r$, which the second agent receives. Since $u=P_{K+\spn\{q\}}f^*$ and $p_r=P_Kf^*$, both residuals $f^*-u$ and $f^*-p_r$ are orthogonal to $K$. Their difference is $u-p_r$, so $u-p_r\perp K$. Also, $p_t-p_r\perp K$, since it is the difference of the residuals $f^*-p_r$ and $f^*-p_t$, both orthogonal to $K\subseteq V_r\subseteq V_t$. Because $p_r\in K$, we can therefore write the full input space as the orthogonal sum
\begin{equation*}
 K+\spn\{p_t,u\}=K+\spn\{p_t-p_r,u-p_r\}.
\end{equation*}

The projection onto an orthogonal sum is the sum of the projections onto its two subspaces. Since the projection onto $K$ is $p_r$, this gives
\begin{equation*}
 g_{\ell,r}=p_r+P_{\spn\{p_t-p_r,u-p_r\}}f^*
 \in\spn\{p_r,p_t,u\}.
\end{equation*}
Thus the second agent can form $g_{\ell,r}$ from its three parents. All of its inputs, including $x_1$, lie in $K+\spn\{p_t,u\}$, where $g_{\ell,r}$ already minimizes the error. It therefore also minimizes the error among the second agent's available predictions, so the agent predicts $g_{\ell,r}$.
\end{proof}

For a longer interval $[\ell,r]$, we split at an index $m$ and use the fits from $[\ell,m]$ and $[m,r]$. The shared prediction $p_m$ lets us separate the two interval spans into a common part and orthogonal remaining parts, as shown in the next lemma.

\begin{lemma}
\label{lem:b-3:interval-overlap}
For $0\le\ell<m<r\le t$, set
\begin{equation*}
 L=W_{\ell,m},\qquad R=W_{m,r},\qquad
 C=\spn\{x_1,p_t,p_m\}.
\end{equation*}
Then
\begin{equation}
\label{eq:b-3:interval-overlap}
 P_{L+R}=P_L+P_R-P_C.
\end{equation}
\end{lemma}
\begin{proof}
Let $T=L\cap C^\perp$ be the subspace of vectors in $L$ orthogonal to $C$. Since $C\subseteq L$, we have the orthogonal decomposition $L=C+T$. We will show that $T$ is also orthogonal to all of $R$. Because $C\subseteq R$, this will give the orthogonal decomposition $L+R=R+T$, and hence
\begin{equation*}
 P_L=P_C+P_T,\qquad P_{L+R}=P_R+P_T.
\end{equation*}
Rewriting $P_T$ in the second equality using $P_T=P_L-P_C$ from the first equality gives the claimed identity. It remains to prove $T\perp R$.

The nested projections imply that $p_j-p_m\perp V_m$ for every $m\le j\le t$. Indeed, both residuals $f^*-p_m$ and $f^*-p_j$ are orthogonal to $V_m\subseteq V_j$, and their difference is $p_j-p_m$. In particular, we can write $L$ as the orthogonal sum
\begin{equation*}
 L=\spn\{x_1,p_\ell,\ldots,p_m\}
   +\spn\{p_t-p_m\}.
\end{equation*}
The first summand lies in $V_m$. The second is orthogonal to $V_m$ and lies in $C$, since $p_t,p_m\in C$. Every vector in $T$ is orthogonal to $C$, so its component in the second summand must be zero. Consequently, $T\subseteq V_m$.

Now take any $w\in T$. Since $w\in V_m$, it is orthogonal to every difference $p_j-p_m$ for $m\le j\le t$. It is also orthogonal to $p_m\in C$. Therefore, for every $m\le j\le t$,
\begin{equation*}
 \ip{w}{p_j}=\ip{w}{p_m}+\ip{w}{p_j-p_m}=0.
\end{equation*}
This includes the predictions $p_m,\ldots,p_r,p_t$ that generate $R$ together with $x_1$. Since $x_1\in C$, we also have $w\perp x_1$. Thus $w\perp R$, proving $T\perp R$ and completing the proof.
\end{proof}

We now turn this identity between spaces into a way to combine their fits. The next agent will receive the fits from both intervals and one more fit from their shared inputs $x_1,p_t,p_m,q$.

\begin{lemma}
\label{lem:b-3:history-merge}
For $0\le\ell<m<r\le t$, an agent observing $x_1$ and receiving $g_{\ell,m}$, $g_{m,r}$, and $g_{m,m}$ predicts $g_{\ell,r}$.
\end{lemma}
\begin{proof}
We will show that $g_{\ell,r}\in\spn\{g_{\ell,m},g_{m,r},g_{m,m}\}$. This suffices because all of the agent's inputs lie in $W_{\ell,r}+\spn\{q\}$, where $g_{\ell,r}$ minimizes the error.

First suppose $\ip{f^*-p_t}{q}=0$. The residual $f^*-p_t$ is orthogonal to $V_t$ and, in this case, to $q$. It is therefore orthogonal to all four spaces $W_{\ell,r}, W_{m,r}, W_{\ell,m}, W_{m,m}$, even after adding $q$. Thus $p_t$ remains the best prediction in each case, giving
\begin{equation*}
 g_{\ell,m}=g_{m,r}=g_{m,m}=g_{\ell,r}=p_t.
\end{equation*}
All three parents already supply $g_{\ell,r}$, which proves the claim in this case.

Suppose $\ip{f^*-p_t}{q}\ne0$. Since $f^*-p_t\perp V_t$, the nonzero inner product with $q$ implies $q\notin V_t$. For each of the four intervals $[i,j]$, set $s_{i,j}=\nrm{q-P_{W_{i,j}}q}^2$, the squared norm of the part of $q$ outside its interval space. Since $q \notin V_t$ and $W_{i,j}\subseteq V_t$, each $s_{i,j}$ is positive. We will show that:
\begin{equation}\label{eq:b-3:interval-g-combination}
 g_{\ell,r}
 =\frac{s_{\ell,m}g_{\ell,m}+s_{m,r}g_{m,r}-s_{m,m}g_{m,m}}{s_{\ell,r}}.
\end{equation}

By showing the above, we conclude that the agent can compute $g_{\ell,r}$ from its three parents and concludes the proof.

We first compute $g_{i,j}$, by separating its input space into two orthogonal subspaces. Since $P_{W_{i,j}}q$ already belongs to $W_{i,j}$, subtracting it from $q$ leaves the input span unchanged:
\begin{equation*}
 W_{i,j}+\spn\{q\}
 =W_{i,j}+\spn\{q-P_{W_{i,j}}q\}.
\end{equation*}
The two subspaces on the right are orthogonal, because the projection residual $q-P_{W_{i,j}}q$ is orthogonal to $W_{i,j}$. We can therefore compute $g_{i,j}$ by adding the projections of $f^*$ onto these two subspaces.

The projection onto $W_{i,j}$ is $p_t$: we have $p_t\in W_{i,j}\subseteq V_t$, and $f^*-p_t\perp V_t$. For the other subspace, we use the projection onto the line spanned by $q-P_{W_{i,j}}q$, whose squared norm is $s_{i,j}>0$. This gives
\begin{equation*}
 g_{i,j}
 =p_t+\frac{\ip{f^*}{q-P_{W_{i,j}}q}}{s_{i,j}}
       (q-P_{W_{i,j}}q).
\end{equation*}

We now simplify the numerator. Since $p_t\in W_{i,j}$, it is orthogonal to $q-P_{W_{i,j}}q$. Also, $f^*-p_t$ is orthogonal to $P_{W_{i,j}}q\in W_{i,j}\subseteq V_t$. These two facts give, respectively,
\begin{equation*}
 \ip{f^*}{q-P_{W_{i,j}}q}
 =\ip{f^*-p_t}{q-P_{W_{i,j}}q}
 =\ip{f^*-p_t}{q}.
\end{equation*}
Substituting this numerator and multiplying by $s_{i,j}$ yields
\begin{equation}\label{eq:b-3:interval-g-formula}
 s_{i,j}g_{i,j}
 =s_{i,j}p_t+\ip{f^*-p_t}{q}(q-P_{W_{i,j}}q).
\end{equation}

We now combine the three weighted parent predictions using this formula. Applying \Cref{lem:b-3:interval-overlap} to $q$ gives $P_{W_{\ell,r}}q=P_{W_{\ell,m}}q+P_{W_{m,r}}q-P_{W_{m,m}}q$. Subtracting both sides from $q$ and grouping the terms gives
\begin{equation}\label{eq:b-3:interval-q-equality}
 q-P_{W_{\ell,r}}q
 =(q-P_{W_{\ell,m}}q)+(q-P_{W_{m,r}}q)-(q-P_{W_{m,m}}q).
\end{equation}

This equality also gives a relation between the weights. For each $[i,j]$, we can split $q$ into its projection $P_{W_{i,j}}q$ and its residual $q-P_{W_{i,j}}q$. These two vectors are orthogonal, so
\begin{align*}
 \ip{q}{q-P_{W_{i,j}}q}
 &=\ip{P_{W_{i,j}}q}{q-P_{W_{i,j}}q}
   +\nrm{q-P_{W_{i,j}}q}^2\\
 &=s_{i,j}.
\end{align*}
Taking inner products with $q$ in \eqref{eq:b-3:interval-q-equality} for $q-P_{W_{\ell,r}}q$ therefore yields
\begin{equation}\label{eq:b-3:s-equality}
 s_{\ell,r}=s_{\ell,m}+s_{m,r}-s_{m,m}.
\end{equation}

Now consider $s_{\ell,m}g_{\ell,m}+s_{m,r}g_{m,r}-s_{m,m}g_{m,m}$ and rewrite it using \eqref{eq:b-3:interval-g-formula} and then \eqref{eq:b-3:s-equality}:
\begin{align*}
 &s_{\ell,m}g_{\ell,m}+s_{m,r}g_{m,r}-s_{m,m}g_{m,m}\\
 &\quad=(s_{\ell,m}+s_{m,r}-s_{m,m})p_t\\
 &\qquad\quad+\ip{f^*-p_t}{q}
   \left((q-P_{W_{\ell,m}}q)+(q-P_{W_{m,r}}q)-(q-P_{W_{m,m}}q)\right)\\
 &\quad=s_{\ell,r}p_t+\ip{f^*-p_t}{q}(q-P_{W_{\ell,r}}q)\\
 &\quad=s_{\ell,r}g_{\ell,r},
\end{align*}
where the last equality uses \eqref{eq:b-3:interval-g-formula} backwards. Since $s_{\ell,r}>0$, dividing by it gives \eqref{eq:b-3:interval-g-combination}.
\end{proof}

We can now assemble the gadget. Each interval will compute its fit from the two shorter intervals and the shared inputs, stopping at the pairs covered by \Cref{lem:b-3:history-leaf}.

\begin{lemma}
\label{lem:b-3:history-extension}
Let $V_0\subseteq\cdots\subseteq V_t$ be nested spaces with $V_0=\spn\{x_1\}$. Suppose agents predict $p_i=P_{V_i}f^*$, with $V_i=\spn\{x_1,p_0,\ldots,p_i\}$ for every $0\le i\le t$. Let $q$ be the prediction of an agent observing $x_1$ and receiving $p_t$. A gadget whose graph and allocation depend only on $t$ computes
\begin{equation*}
 P_{V_t+\spn\{q\}}f^*.
\end{equation*}
Every added agent observes $x_1$ and has at most three parents. For $t\ge2$, the gadget adds at most $4(t-1)$ agents and has $O(\log t)$ additional depth. For $t=0$ or $t=1$, $q$ already equals this prediction, so no agents are added.
\end{lemma}
\begin{proof}
For $t=0$ or $t=1$, the hypothesis on $V_t$ and the fact that $p_0\in\spn\{x_1\}$ give $V_t=\spn\{x_1,p_t\}$. The agent predicting $q$ observes $x_1$ and receives $p_t$, so its residual $f^*-q$ is orthogonal to $x_1$ and $p_t$. The residual is also orthogonal to $q$, because $q$ belongs to that agent's input span. Therefore $q$ belongs to $V_t+\spn\{q\}$ and its residual is orthogonal to this entire space. This proves $q=P_{V_t+\spn\{q\}}f^*$ for these two values of $t$.

Suppose $t\ge2$. We recursively construct an agent predicting $g_{\ell,r}$ for each interval used in the construction. If $r=\ell+1$, use the agents in \Cref{lem:b-3:history-leaf}. For $r-\ell\ge2$, set $m=\lfloor(\ell+r)/2\rfloor$ and construct $g_{\ell,m}$ and $g_{m,r}$ in parallel. Add an agent observing $x_1$ with parents $p_t,p_m,q$. It predicts $P_{\spn\{x_1,p_t,p_m,q\}}f^*=g_{m,m}$, the third parent prediction required by \Cref{lem:b-3:history-merge}. An agent observing $x_1$ and receiving this prediction and the two interval fits then predicts $g_{\ell,r}$. Induction on $r-\ell$ proves that the root computes $g_{0,t}=P_{V_t+\spn\{q\}}f^*$.

The recursion has $t$ leaves, one for each pair $[i,i+1]$, and $t-1$ internal nodes, since every internal node has two children. By \Cref{lem:b-3:history-leaf}, the two end pairs use one agent each and the other $t-2$ pairs use two each. Every internal node uses two agents: one to predict $g_{m,m}$ and one to combine the three predictions using \Cref{lem:b-3:history-merge}. Thus the number of added agents is
\begin{equation*}
 2+2(t-2)+2(t-1)=4(t-1).
\end{equation*}
The leaves need at most two layers. At every internal node, the agent with parents $p_t,p_m,q$ can be computed in parallel with the child intervals, so merging the three predictions adds one more layer. Splitting each interval at its midpoint gives at most $\lceil\log_2t\rceil$ levels of merges. The added depth is therefore at most $2+\lceil\log_2t\rceil=O(\log t)$.

All added agents observe $x_1$ and have at most three parents. The interval splits, the parent choices, and the choice between the one-agent and two-agent constructions do not depend on the distribution. Hence the same gadget works for every distribution with finite second moments.
\end{proof}

We finish by showing that the fixed graph reaches exact aggregation after the prescribed $d-1$ rounds. The key point is that every round before exact aggregation increases the dimension of the space retained by the predictions.

\fixedthreeparents*
\begin{proof}
For $d=1$, the global feature space is $\spn\{x_1\}$, so a source observing $x_1$ already predicts $f^*$. Suppose $d\ge2$, and use the source and the $d-1$ rounds described in \Cref{subsec:b-3:fixed}. Initially,
\begin{equation*}
 V_0=\spn\{x_1\}=\spn\{x_1,p_0\},\qquad p_0=P_{V_0}f^*.
\end{equation*}
Assume that after $t$ rounds, for every $0\le i\le t$, the available predictions satisfy $p_i=P_{V_i}f^*$ and $V_i=\spn\{x_1,p_0,\ldots,p_i\}$, with the spaces nested. We show that the next round preserves these properties.

First suppose $p_t\ne f^*$. The residual $f^*-p_t$ is a nonzero vector in $H$ and is orthogonal to $V_t$. Since the raw features span $H$, some feature has a nonzero inner product with $f^*-p_t$. Otherwise this residual would be orthogonal to all of $H$, including itself. This feature cannot be $x_1\in V_t$, so call it $x_i$ with $i\ge2$. The agent observing $x_i$ and receiving $p_t$ can use any predictor $p_t+\alpha x_i$. Choosing $\alpha=\ip{f^*-p_t}{x_i}/\nrm{x_i}^2$ gives error
\begin{equation*}
 \nrm{f^*-p_t-\alpha x_i}^2
 =\nrm{f^*-p_t}^2-\frac{|\ip{f^*-p_t}{x_i}|^2}{\nrm{x_i}^2}
 <\nrm{f^*-p_t}^2.
\end{equation*}
Its prediction has a smaller or equal error, so at least one of the agents testing a raw feature improves strictly on $p_t$. Each agent in the binary tree can use either child's prediction, so the error cannot increase on the path from that improving agent to the root. Hence the root prediction $q_t$ also improves strictly on $p_t$.

The root observes $x_1$ and receives $p_t$. Thus \Cref{lem:b-3:history-extension} applies and computes
\begin{equation*}
 V_{t+1}=V_t+\spn\{q_t\},\qquad p_{t+1}=P_{V_{t+1}}f^*.
\end{equation*}
The induction hypothesis gives $V_{t+1}=\spn\{x_1,p_0,\ldots,p_t,q_t\}$, the first equality in \eqref{eq:b-3:invariant}. Since $p_t$ is the best prediction in $V_t$ and $q_t$ has smaller error, $q_t\notin V_t$, so $V_{t+1}$ extends $V_t$ by one dimension. The fit $p_{t+1}$ has no larger error than $q_t$, which also forces $p_{t+1}\notin V_t$. Consequently, $V_t+\spn\{p_{t+1}\}$ is a subspace of $V_{t+1}$ with the same dimension. These spaces are equal, proving the second equality in \eqref{eq:b-3:invariant} and completing the induction in this case.

If $p_t=f^*$, every agent testing a raw feature has $f^*$ among its inputs and therefore predicts it with zero error. The same holds throughout the binary tree, so $q_t=f^*$. The gadget then returns $p_{t+1}=f^*$ and $V_{t+1}=V_t$, which preserves the induction hypothesis in this case as well.

All predictions lie in $H$, so the nested spaces also lie in $H$. If $x_1\ne0$, then $\dim V_0=1$ and $\dim H\le d$. Thus there can be at most $d-1$ dimension increases before the retained space equals $H$, at which point its fit is $f^*$. If $x_1=0$, then $\dim V_0=0$ and $\dim H\le d-1$, giving the same bound. Since every round before exact aggregation increases the dimension, and every later round retains $f^*$, the prescribed $d-1$ rounds end with $p_{d-1}=f^*$.

Each improvement tree uses $d-1$ agents to test the features $x_2,\ldots,x_d$ and $d-1$ internal agents to combine their predictions with $p_t$. The gadget adds no agents for $t\le1$ and at most $4(t-1)$ otherwise, by \Cref{lem:b-3:history-extension}. Including the source and all $d-1$ rounds gives at most
\begin{equation*}
 1+2(d-1)^2+4\sum_{t=2}^{d-2}(t-1)\le4d^2
\end{equation*}
agents, taking the sum to be zero when $d\le3$. The agents testing the features add one layer, and the balanced binary tree adds $O(\log d)$ layers. The gadget adds $O(\log t)$ layers for $t\ge2$ and none otherwise. Since $t<d$, every round adds $O(\log d)$ depth, giving total depth $O(d\log d)$.
\end{proof}

\subsection{Selecting features and updating predictions}
\label{app:exact:feature-update}

We now prove \Cref{thm:b-3:exact}. For $f^*\ne0$, choose linearly independent features $x_1,\ldots,x_r$ spanning $H$. For some set $J_t\subseteq[r]$ of $t$ features which we choose later, define
\begin{equation*}
 H_t=\spn\{x_j:j\in J_t\},\qquad
 p_t=P_{H_t}f^*,\qquad
 q_{t,i}=P_{H_t+\spn\{x_i\}}f^*\quad(i\notin J_t).
\end{equation*}
The prediction $p_t$ uses the selected features, and $q_{t,i}$ is the prediction obtained when $x_i$ is also available. We start with $J_0=\varnothing$, $H_0=\{0\}$, and $p_0=0$. If $p_t=f^*$, we stop. Otherwise, select an index $j\notin J_t$ with the smallest nonzero value of $\nrm{q_{t,j}-p_t}^2$, and set $J_{t+1}=J_t\cup\{j\}$ and $p_{t+1}=q_{t,j}$. For $t\ge1$, the agent predicting $q_{t,j}$ will already exist, so this selection requires no new agent. The remaining task is to produce $q_{t+1,i}$ for every $i\notin J_{t+1}$.

For $t\ge1$, each new agent will observe $x_i$ and receive $p_t,p_{t+1},q_{t,i}$. We handle the first update separately in the theorem proof below. If $q_{t,i}=p_t$, that parent supplies no additional vector, so we need to express $q_{t+1,i}$ using $x_i,p_t,p_{t+1}$. To make this update exact, we maintain the condition
\begin{equation}
\label{eq:exact:zero-improvement}
 q_{t,i}=p_t\quad\Longrightarrow\quad x_i\perp H_t
 \qquad(i\notin J_t).
\end{equation}
The condition holds initially because $H_0=\{0\}$. The next lemma shows that our selection rule preserves it and puts each required prediction in the new agent's input span. Its proof also explains why we choose the smallest nonzero improvement.

\begin{lemma}
\label{lem:exact:feature-update}
Suppose $p_t\ne f^*$ and \eqref{eq:exact:zero-improvement} holds. Then $q_{t,i}\ne p_t$ for some $i\notin J_t$. Choose $j\notin J_t$ minimizing $\nrm{q_{t,j}-p_t}^2$ among its nonzero values, and set $J_{t+1}=J_t\cup\{j\}$ and $p_{t+1}=q_{t,j}$. For every $i\notin J_{t+1}$,
\begin{equation*}
 q_{t+1,i}\in
 \begin{cases}
 \spn\{p_t,p_{t+1},q_{t,i}\},&q_{t,i}\ne p_t,\\
 \spn\{p_t,p_{t+1},x_i\},&q_{t,i}=p_t.
 \end{cases}
\end{equation*}
Moreover, \eqref{eq:exact:zero-improvement} holds with $t+1$ in place of $t$.
\end{lemma}
\begin{proof}
We first show that a selection is possible. The residual $f^*-p_t$ is a nonzero vector in $H$, so it has a nonzero inner product with some feature among $x_1,\ldots,x_r$. Otherwise the residual would be orthogonal to all of $H$, including itself. That feature must be unselected, since the residual is orthogonal to $H_t$. Call its index $i$. Since $f^*-q_{t,i}\perp x_i$ but $\ip{f^*-p_t}{x_i}\ne0$, the predictions $q_{t,i}$ and $p_t$ must differ and hence $\nrm{q_{t,i}-p_t}^2>0$.

To compare the predictions before and after selecting $j$, separate each remaining feature into its part in $H_t$ and its part orthogonal to $H_t$. For every $i\notin J_t$, define
\begin{equation*}
 u_i=x_i-P_{H_t}x_i.
\end{equation*}
The vectors $u_i$ are linearly independent. Indeed, any nontrivial linear relation among them would express a nontrivial combination of the unselected features as a combination of the selected features, contradicting the independence of $x_1,\ldots,x_r$.

We next prove the two span claims. Since $x_i-u_i=P_{H_t}x_i\in H_t$, we have $H_t+\spn\{x_i\}=H_t+\spn\{u_i\}$. The latter sum is orthogonal, and the projection of $q_{t,i}$ onto $H_t$ is $p_t$, so
\begin{equation*}
 q_{t,i}-p_t=P_{\spn\{u_i\}}f^*.
\end{equation*}
Whenever $q_{t,i}\ne p_t$, the difference $q_{t,i}-p_t$ is therefore a nonzero multiple of $u_i$. In particular, the selected prediction satisfies $p_{t+1}-p_t=q_{t,j}-p_t\ne0$.

Now fix $i\notin J_{t+1}$. Since $H_{t+1}=H_t+\spn\{x_j\}$ and both $x_j-u_j$ and $x_i-u_i$ lie in $H_t$, the space defining $q_{t+1,i}$ can be written as
\begin{equation*}
 H_{t+1}+\spn\{x_i\}=H_t+\spn\{u_j,u_i\}.
\end{equation*}
Both $u_j$ and $u_i$ are orthogonal to $H_t$, so projection onto this space gives
\begin{equation}
\label{eq:exact:two-directions}
 q_{t+1,i}
 =p_t+P_{\spn\{u_j,u_i\}}f^*.
\end{equation}
If $q_{t,i}\ne p_t$, the two differences $p_{t+1}-p_t$ and $q_{t,i}-p_t$ span the same space as $u_j,u_i$. By \eqref{eq:exact:two-directions}, adding $p_t$ to a linear combination of these differences gives $q_{t+1,i}$, so
\begin{equation*}
 q_{t+1,i}\in\spn\{p_t,p_{t+1},q_{t,i}\}.
\end{equation*}
If $q_{t,i}=p_t$, the condition \eqref{eq:exact:zero-improvement} gives $x_i\perp H_t$, so $u_i=x_i$. We can then use $p_{t+1}-p_t$ and $x_i$ to span $u_j,u_i$. The same formula gives
\begin{equation*}
 q_{t+1,i}\in\spn\{p_t,p_{t+1},x_i\},
\end{equation*}
which proves the other span claim.

It remains to prove \eqref{eq:exact:zero-improvement} after the selection. First take $i\notin J_{t+1}$ with $q_{t,i}\ne p_t$. We will show that $q_{t+1,i}\ne p_{t+1}$, using the smallest nonzero improvement rule. Since $p_t$ belongs to the input spaces of both $q_{t,i}$ and $p_{t+1}$, \Cref{lem:least-squares-error} and the choice of $j$ give
\begin{align*}
 \nrm{f^*-q_{t,i}}^2
 &=\nrm{f^*-p_t}^2-\nrm{q_{t,i}-p_t}^2\\
 &\le\nrm{f^*-p_t}^2-\nrm{p_{t+1}-p_t}^2\\
 &=\nrm{f^*-p_{t+1}}^2.
\end{align*}
Thus $q_{t,i}$ has no greater error than $p_{t+1}$. These predictions are distinct: their differences from $p_t$ are nonzero multiples of the independent vectors $u_i,u_j$. Also, $q_{t,i}$ belongs to $H_{t+1}+\spn\{x_i\}$, since it belongs to $H_t+\spn\{x_i\}$ and $H_t\subseteq H_{t+1}$. The best prediction in a subspace is unique, so $p_{t+1}$ cannot be optimal in this space when a distinct available prediction has no greater error. Therefore its best prediction $q_{t+1,i}$ differs from $p_{t+1}$, as claimed.

Now suppose $q_{t+1,i}=p_{t+1}$. We just proved above that $q_{t,i}\ne p_t$ implies $q_{t+1,i}\ne p_{t+1}$, so $q_{t,i}=p_t$ here. Both $q_{t,i}$ and $q_{t+1,i}$ are projections onto spaces containing $x_i$, so
\begin{equation*}
 \ip{f^*-p_t}{x_i}=0,
 \qquad \ip{f^*-p_{t+1}}{x_i}=0.
\end{equation*}
Subtracting gives $\ip{p_{t+1}-p_t}{x_i}=0$. Since $p_{t+1}-p_t$ is a nonzero multiple of $u_j$, this implies $x_i\perp u_j$. The equality $q_{t,i}=p_t$ and \eqref{eq:exact:zero-improvement} also give $x_i\perp H_t$. Hence $x_i\perp H_{t+1}$, since $H_{t+1}=H_t+\spn\{u_j\}$. This proves the condition at the next step.
\end{proof}

We now use \Cref{lem:exact:feature-update} to construct the agents and prove the bounds. Only the first selected feature needs a source. All other predictions will be obtained from this source and the agents added after each selection.

\adaptivethreeparents*
\begin{proof}
If $f^*=0$, every agent predicts zero by \Cref{lem:replace-y-fstar}, so one source suffices. Suppose $f^*\ne0$. We follow the selection rule above and stop as soon as a selected prediction equals $f^*$, using its agent as the output.

Since $H_0=\{0\}$, \eqref{eq:exact:zero-improvement} holds initially. The designer computes the values $q_{0,i}$ to make the first selection, which exists by \Cref{lem:exact:feature-update}, and creates one source observing the selected feature $x_j$. Its prediction is $p_1=q_{0,j}$.

Unless we have stopped, create each $q_{1,i}$ using an agent observing $x_i$ and receiving only $p_1$. Since $p_1$ is a nonzero multiple of $x_j$, its input space is $\spn\{x_i,p_1\}=H_1+\spn\{x_i\}$, so it predicts $q_{1,i}$. By \Cref{lem:exact:feature-update} at $t=0$, \eqref{eq:exact:zero-improvement} holds after this first selection.

After each subsequent selection $t\ge2$, unless we have stopped, for every $i\notin J_t$, add an agent observing $x_i$ and receiving $p_{t-1},p_t,q_{t-1,i}$. By \Cref{lem:exact:feature-update}, applied at step $t-1$, the inputs span $q_{t,i}$ and \eqref{eq:exact:zero-improvement} is preserved. All inputs lie in $H_t+\spn\{x_i\}$, where $q_{t,i}$ is the best prediction, so the agent predicts $q_{t,i}$. Each selection after the first reuses the agent predicting $q_{t,j}$ as the agent predicting $p_{t+1}$.

Each update adds at most one layer, while selecting $p_{t+1}=q_{t,j}$ adds no agent. Starting with $p_1$ at depth one, induction gives $p_t$ by depth $t$ and $q_{t,i}$ by depth $t+1$.

Every selection adds one of the $r$ basis features. If the construction has not stopped earlier, after $r$ selections we have $H_r=H$ and $p_r=f^*$, so the output has depth at most $r$. The construction uses one source and at most $r-t$ new agents after selection $t$, for $1\le t<r$. Its size is therefore at most $1+\sum_{t=1}^{r-1}(r-t)=1+\binom r2$.
\end{proof}

\subsection{The feature order required along a path}
\label{app:b-3:ordered-path}

The following proof refines the propagation argument underlying \citet[Theorem~5.9]{kearns2026networked}. Their theorem gives a depth barrier for the same Gaussian construction without normalization. We track the feature order along paths and include the normalization in the error bound.

\begin{proposition}
\label{lem:b-3:ordered-path}
For the distribution in \eqref{eq:b-3:ordered-features}, fix any graph, single-feature allocation, and output agent $A_G$. Let $\ell$ be the largest prefix length such that some directed path ending at $A_G$ contains agents observing $x_1,\ldots,x_\ell$ in this order, possibly with other agents between them. If $\ell<d$, then
\begin{equation*}
 \MSE(f_G)-\MSE(f^*)\ge\frac1{2d^2(\ell+1)}.
\end{equation*}
\end{proposition}
\begin{proof}
For each agent $A_v$, let $k(A_v)$ be the length of the longest prefix of the feature order $x_1,\ldots,x_d$ appearing along a directed path ending at $A_v$. We first prove by induction in a topological order that its prediction lies in $\spn\{x_1,\ldots,x_{k(A_v)}\}$.

Let $L$ be the maximum prefix length among the parents of $A_v$, taking $L=0$ if it has no parents. By induction, all parent predictions lie in the span of the first $L$ features. If $L=d$, the induction claim follows because every prediction lies in $H$. Suppose $L<d$. A raw feature $x_j$ with $j>L+1$ involves only Gaussian variables $Z_{L+1},\ldots,Z_{d-1}$. It is therefore orthogonal to both $Y$ and the first $L$ feature span. Adding it to the parents' inputs does not change the projection of $Y$. A feature with $j\le L$ is already in that span. Finally, observing $x_{L+1}$ extends a path containing the first $L$ features in order to one containing the first $L+1$ in order, and the new prediction lies in the first $L+1$ feature span. These cases conclude the induction.

At the output, $k(A_G)=\ell$, so the induction gives $f_G\in\spn\{x_1,\ldots,x_\ell\}$. It remains to bound the error of predictions in this space when $\ell<d$. We do this by computing $P_{\spn\{x_1,\ldots,x_\ell\}}Y$ which is the best prediction in the space $\spn\{x_1,\ldots,x_\ell\}$ for $Y$, and thus it is not a worse prediction than $f_G$ because $f_G$ belongs to the same space.

The variables $Z_0,\ldots,Z_\ell$ are orthonormal in $L^2(\gD)$, since they are independent standard Gaussians. Within their span, a vector is orthogonal to $x_i=(Z_{i-1}-Z_i)/\sqrt2$ exactly when its coefficients on $Z_{i-1}$ and $Z_i$ are equal. Orthogonality to all of $x_1,\ldots,x_\ell$ therefore requires all $\ell+1$ coefficients to be equal. Thus the vectors in this Gaussian span orthogonal to the first $\ell$ features form the line spanned by
\begin{equation*}
 w=Z_0+\cdots+Z_\ell.
\end{equation*}

The label $Y=Z_0/(\sqrt2\,d)$ also lies in $\spn\{Z_0,\ldots,Z_\ell\}$. Its residual after projection onto the first $\ell$ features is therefore its projection onto the line spanned by $w$. Orthonormality gives $\nrm{w}^2=\ell+1$ and $\ip{Y}{w}=1/(\sqrt2\,d)$, so
\begin{equation*}
 Y-P_{\spn\{x_1,\ldots,x_\ell\}}Y
 =\frac{\ip{Y}{w}}{\nrm{w}^2}w
 =\frac{w}{\sqrt2\,d(\ell+1)}.
\end{equation*}
Since $P_{\spn\{x_1,\ldots,x_\ell\}}Y$ minimizes the error over $\spn\{x_1,\ldots,x_\ell\}$ and $f_G$ belongs to this space, we obtain
\begin{equation*}
 \MSE(f_G)
 \ge\nrm{Y-P_{\spn\{x_1,\ldots,x_\ell\}}Y}^2
 =\frac{\nrm{w}^2}{2d^2(\ell+1)^2}
 =\frac1{2d^2(\ell+1)}.
\end{equation*}
Finally, $f^*=Y$ has zero error, so this is also the claimed lower bound on the excess error.
\end{proof}

\depthlowerbounds*
\begin{proof}
For the adaptive designer setting, use the normalized distribution in \eqref{eq:b-3:ordered-features}. It has rank $d$ because the coefficient matrix of the features in $Z_0,\ldots,Z_{d-1}$ is triangular with nonzero diagonal. By \Cref{lem:b-3:ordered-path}, exact aggregation requires a path containing all $d$ features, hence at least $d$ agents. For $D<d$, every path in a depth-$D$ network has at most $D$ agents, so \Cref{lem:b-3:ordered-path} gives excess error at least $1/(2d^2(D+1))$. This gives $\siR{b}{d}{D}>0$ for every $b\ge1$.

For the oblivious designer setting, fix $b\ge2$ and a graph, single-feature allocation, and output agent that achieve exact aggregation for every distribution on $d$ features, with at most $b$ parents per agent. Let $D$ be the output depth.

An agent at depth $s$ has at most $b^{s-1}$ paths from sources to it. A source has one such path. At any other agent, their number is the sum of the path counts at its at most $b$ parents, each of depth at most $s-1$. This proves the bound by induction.

There are therefore at most $b^{D-1}$ source-to-output paths. Each has at most $D$ agents, so choosing $d$ positions on it gives at most $\binom Dd$ possible feature orders. Every path ending at the output can be extended backward to a source. Applying \Cref{lem:b-3:ordered-path} after each relabeling of the features requires these source-to-output paths to contain every permutation of $[d]$. Counting the $d!$ permutations gives
\begin{equation}
\label{eq:b-3:path-count}
 b^{D-1}\binom Dd\ge d!.
\end{equation}

Finally, $\binom Dd\le2^D$, so \eqref{eq:b-3:path-count} implies $(2b)^D\ge d!$. Taking logarithms and using $\log(d!)\ge d\log d-d$ gives the stated order bound.
\end{proof}

\section{Proofs for two parents per agent}
\label{app:two-parents}

\subsection{Combining three predictions exactly}
\label{app:b-2:three-predictions}

For points $a,b$ in a Euclidean plane, write $F(a,b)$ for the point closest to zero on the line through $a,b$, with $F(a,a)=a$. Minimizing the squared norm of $a+\lambda(b-a)$ over $\lambda\in\R$ gives
\begin{equation}
\label{eq:b-2:foot}
 F(a,b)=a-\frac{\ip{a}{b-a}}{\nrm{b-a}^2}(b-a)
 \qquad(a\ne b).
\end{equation}
Thus, when $a\ne b$, $F(a,b)$ is the unique point on the line through $a,b$ that is orthogonal to its direction $b-a$.

To turn pairwise fits into operations on points in a plane, we rescale each prediction so that its component along the desired prediction $g$ is exactly $g$. Subtracting $g$ then leaves a point orthogonal to it. The next lemma shows that a pairwise fit becomes an application of $F$, and the desired prediction becomes zero.

\begin{lemma}
\label{lem:b-2:prediction-map}
Let $V\subseteq L^2(\gD)$ have dimension three, and let $g\in V$ be nonzero. For every nonzero $f\in V$ satisfying $\ip{g}{f}=\nrm{f}^2$, define
\begin{equation}
\label{eq:b-2:prediction-point}
 n(f)=\frac{\nrm{g}^2}{\nrm{f}^2}f-g.
\end{equation}
Then $n(f)\in V\cap\spn\{g\}^\perp$ and
\begin{equation}
\label{eq:b-2:point-prediction}
 f=\frac{\nrm{g}^2}{\nrm{g}^2+\nrm{n(f)}^2}\bigl(g+n(f)\bigr).
\end{equation}
In particular, $n(f)=0$ exactly when $f=g$. For any two such vectors $f_1,f_2$, their fit $q=P_{\spn\{f_1,f_2\}}g$ is nonzero and satisfies
\begin{equation*}
 n(q)=F\bigl(n(f_1),n(f_2)\bigr).
\end{equation*}
\end{lemma}
\begin{proof}
The vector $n(f)$ belongs to $V$, and the hypothesis on $f$ gives
\begin{equation*}
 \ip{g}{n(f)}
 =\frac{\nrm{g}^2}{\nrm{f}^2}\ip{g}{f}-\nrm{g}^2=0.
\end{equation*}
Taking squared norms in $g+n(f)=\nrm{g}^2f/\nrm{f}^2$ therefore gives $\nrm{g}^2+\nrm{n(f)}^2=\nrm{g}^4/\nrm{f}^2$. Substitution into the same equality proves \eqref{eq:b-2:point-prediction}. This formula gives $f=g$ when $n(f)=0$, and the definition gives $n(g)=0$.

For the pairwise fit, set $a=n(f_1)$ and $b=n(f_2)$. If $a=b$, \eqref{eq:b-2:point-prediction} gives $f_1=f_2$. The hypothesis $\ip{g}{f_1}=\nrm{f_1}^2$ then gives $q=f_1$, so $n(q)=a=F(a,a)$.

Suppose $a\ne b$, and let $c=F(a,b)$. The line through $a,b$ also passes through $c$, so
\begin{equation*}
 \spn\{f_1,f_2\}
 =\spn\{g+a,g+b\}
 =\spn\{g+c\}+\spn\{b-a\}.
\end{equation*}
For the last equality, $g+c$ is a linear combination of $g+a,g+b$, and each of $g+a,g+b$ differs from $g+c$ by a multiple of $b-a$. The two spaces on the right are orthogonal: $g\perp b-a$ because $a,b\in\spn\{g\}^\perp$, and $c\perp b-a$ by the definition of $F$. Since $g$ is orthogonal to the second space, its projection onto the input span is
\begin{equation*}
 q=P_{\spn\{g+c\}}g
 =\frac{\nrm{g}^2}{\nrm{g}^2+\nrm{c}^2}(g+c).
\end{equation*}
This vector is nonzero and has squared norm $\nrm{g}^4/(\nrm{g}^2+\nrm{c}^2)$. Substituting into the definition of $n(q)$ gives $n(q)=c=F(a,b)$.
\end{proof}

The space $V\cap\spn\{g\}^\perp$ in \Cref{lem:b-2:prediction-map} has dimension two because $\dim V=3$ and $g\ne0$. An orthonormal basis identifies its two coordinates with the real and imaginary parts of a complex number in $\sC$, with $\ip{z}{w}=\Real (z\overline w)$ and norm $|z|$. Multiplication by a nonzero complex number rotates and scales both lines and distances to zero, so
\begin{equation}
\label{eq:b-2:complex-foot}
 F(\eta z,\eta w)=\eta F(z,w).
\end{equation}
The identity also holds for $\eta=0$.

We continue in the complex plane, starting from the three points associated with the input predictions $f_1,f_2,f_3$. Each new point must be obtained by applying $F$ to two available points. Our goal is to produce zero, which corresponds to the desired prediction $g$.

We organize the construction by keeping triples obtained from the initial triple by a common rotation and scaling. For a triple $T=(a,b,c)$ and a complex number $\eta$, we write $\eta T=(\eta a,\eta b,\eta c)$. The following lemma shows how to combine two sequences that produce such triples.

\begin{lemma}
\label{lem:b-2:copy-operations}
Let $T=(a,b,c)$ be a triple of nonzero complex points, and write $zT=(za,zb,zc)$. Suppose fixed sequences of $F$ operations starting from $T$ produce $\alpha T$ and $\beta T$. Running the first sequence on the output of the second produces $\alpha\beta T$, using the sum of their numbers of operations. Alternatively, the three operations
\begin{equation*}
 F(\alpha a,\beta a),\qquad
 F(\alpha b,\beta b),\qquad
 F(\alpha c,\beta c)
\end{equation*}
produce $F(\alpha,\beta)T$, using that sum plus three operations.
\end{lemma}
\begin{proof}
By \eqref{eq:b-2:complex-foot}, multiplying all three inputs of a fixed sequence by $\beta$ multiplies every intermediate point and every output by $\beta$. This follows by induction over its $F$ operations. Thus the sequence that produces $\alpha T$ from $T$ produces $\alpha\beta T$ when run on $\beta T$.

For the second construction, run both sequences from $T$. Applying \eqref{eq:b-2:complex-foot} to the three stated operations gives the triple
\begin{equation*}
 \bigl(aF(\alpha,\beta),\ bF(\alpha,\beta),\ cF(\alpha,\beta)\bigr)
 =F(\alpha,\beta)T.
\end{equation*}
The first construction uses the operations of both sequences. The second uses those operations and three more.
\end{proof}

Using the above definition, we can work with complex numbers directly rather than triples. Namely, we can represent a complex number $z$ by a corresponding triple $zT$. The next lemma shows that we can also multiply and divide complex numbers, through the operations on triples. More specifically, to support division, we will represent each complex number $z$ by a numerator and denominator triple $NT$ and $DT$, with $z=N/D$. The next lemma shows that we can multiply and divide complex numbers, through the operations on triples.

\begin{lemma}
\label{lem:b-2:quotient-copies}
Let $T=(a,b,c)$ be a triple of nonzero complex points, and suppose fixed sequences of $F$ operations starting from $T$ produce $\alpha T$ and $\beta T$. Every expression $z$ formed from $1,\alpha,\beta$ using multiplication, division by nonzero values, and $F$ has two fixed sequences of $F$ operations producing
\begin{equation*}
 NT=(Na,Nb,Nc),\qquad DT=(Da,Db,Dc),
 \qquad D\ne0,\qquad z=N/D.
\end{equation*}
If $z=0$, the constructed triple $NT$ consists of three zero points.
\end{lemma}
\begin{proof}
For the value $1$, use $T$ as both the numerator and denominator triple. For $\alpha$, use $\alpha T$ as the numerator and $T$ as the denominator, and do the same for $\beta$. We extend these choices through the expression using the following three rules.

To multiply two values $N_1/D_1$ and $N_2/D_2$, construct the triples $(N_1N_2)T$ and $(D_1D_2)T$ by composing sequences as in \Cref{lem:b-2:copy-operations}. They are the numerator and denominator triples for the product because
\begin{equation*}
 \frac{N_1}{D_1}\frac{N_2}{D_2}
 =\frac{N_1N_2}{D_1D_2}.
\end{equation*}
The new denominator is nonzero since $D_1,D_2\ne0$.

To take the reciprocal of a nonzero value $N/D$, exchange the two triples and their sequences. The numerator becomes $DT$ and the denominator becomes $NT$, giving the ratio $D/N$. The new denominator is nonzero because $N/D\ne0$. Division by a nonzero value is multiplication by its reciprocal, so it uses only these two rules.

To apply $F$ to two values, \eqref{eq:b-2:complex-foot} gives
\begin{equation}
\label{eq:b-2:quotient-fit}
 F\left(\frac{N_1}{D_1},\frac{N_2}{D_2}\right)
 =\frac{F(N_1D_2,N_2D_1)}{D_1D_2}.
\end{equation}
First construct $(N_1D_2)T$ and $(N_2D_1)T$ by the multiplication rule. The three $F$ operations in \Cref{lem:b-2:copy-operations} then produce $F(N_1D_2,N_2D_1)T$, the numerator triple. Another use of the multiplication rule produces the denominator triple $(D_1D_2)T$, whose coefficient is nonzero.

These rules give the required sequences for every expression. Finally, $N/D=0$ with $D\ne0$ implies $N=0$, so each entry of the numerator triple is zero. Thus a zero ratio gives an actual zero point without performing division.
\end{proof}

Next we construct the first two triples, $T$ and $\mu T$, from the three given points. These two triples are our only building blocks for the rest of the construction.

\begin{lemma}
\label{lem:b-2:scaled-copy}
Let $T_0=(a_0,b_0,c_0)$ be three noncollinear points in $\sC$, ordered so that $|a_0|\ge\max\{|b_0|,|c_0|\}$. Starting from $T_0$, either zero is available after at most five $F$ operations, or two operations give a noncollinear triple $T=(a,b,c)$ of nonzero points from which a fixed sequence of three $F$ operations produces $\mu T=(\mu a,\mu b,\mu c)$ for some $\mu$ satisfying
\begin{equation*}
 F(1,\mu)=\mu,\qquad \mu\notin\R,\qquad 0<|\mu|<1.
\end{equation*}
\end{lemma}
\begin{proof}
Start from the ordered triple $T_0=(a_0,b_0,c_0)$, and stop if any initial or newly produced point is zero. Set $a=a_0$ and form $T=(a,b,c)$ using
\begin{equation*}
 b=F(a,b_0),\qquad c=F(b,c_0).
\end{equation*}
Then perform the three operations
\begin{equation}
\label{eq:b-2:chain-copy}
 a'=F(a,c),\qquad b'=F(b,a'),\qquad c'=F(c,b').
\end{equation}
We will show that the first two operations give a noncollinear $T$ and the last three produce $\mu T$, with $\mu=a'/a$ satisfying the stated properties.

The ordering of $T_0$ ensures $b\ne a$. Otherwise, $F(a,b_0)=a$ would give $a\perp b_0-a$, so $|b_0|^2=|a|^2+|b_0-a|^2>|a|^2$, contrary to the choice of $a_0$. Since $b$ lies on the line through $a,b_0$ and differs from $a$, the points $a,b,c_0$ remain noncollinear.

We also have $c\ne b$. If $c=b$, the definition of $F$ would give $b\perp c_0-b$. We already have $b\perp a-b$ from the first operation. These two directions are independent because $a,b,c_0$ are noncollinear, so $b$ would be zero, in which case we would have stopped. Since $c$ lies on the line through $b,c_0$ and differs from $b$, the triple $T=(a,b,c)$ is noncollinear. The perpendicular relations $b\perp a-b$ and $c\perp b-c$ also give
\begin{equation}
\label{eq:b-2:chain}
 F(a,b)=b,\qquad F(b,c)=c,\qquad
 0<|c|<|b|<|a|,
\end{equation}
where the strict norm inequalities follow from Pythagoras and $b\ne a$, $c\ne b$.

Now consider $a'=F(a,c)$ and $\mu=a'/a$. The relation $a'\perp a-a'$ when divided by $a$ gives $F(1,\mu)=\mu$, hence $\ip{1-\mu}{\mu}=\Real \mu-|\mu|^2=0$, and thus $\Real \mu=|\mu|^2$. Since $a'$ is the closest point to zero on a line containing $c$, we have $|a'|\le|c|$. Thus
\begin{equation*}
 0<|\mu|=\frac{|a'|}{|a|}\le\frac{|c|}{|a|}<1.
\end{equation*}
A real number satisfying $\Real \mu=|\mu|^2$ must be zero or one, so $\mu$ is nonreal.

It remains to show that $b'=\mu b$ and $c'=\mu c$. We use the following identity: if distinct $z,w\in\sC$ satisfy $F(1,z)=z$ and $F(1,w)=w$, then
\begin{equation}
\label{eq:b-2:circle-product}
 F(z,w)=zw.
\end{equation}
Indeed, $F(1,z)=z$ gives $\Real z=|z|^2$, including when $z=1$, and likewise for $w$. If $zw=0$, one input is zero and the identity follows. Otherwise,
\begin{equation*}
 \ip{zw}{z}=|z|^2\Real w=|zw|^2,
 \qquad
 \ip{zw}{w}=|w|^2\Real z=|zw|^2.
\end{equation*}
Thus $z,w$ lie on the line through $zw$ perpendicular to $zw$. Since they are distinct, this is their line, and its closest point to zero is $zw$.

To apply this identity to $b'$, we have $F(1,b/a)=b/a$ by \eqref{eq:b-2:chain}, and $F(1,\mu)=\mu$ as proved above. These two numbers are distinct: $a'$ lies on the line through $a,c$, whereas $b$ does not because $T$ is noncollinear. Consequently,
\begin{equation*}
 b'=aF(b/a,a'/a)=a\,(b/a)(a'/a)=\mu b.
\end{equation*}
For $c'$, the two numbers $c/b$ and $b'/b=\mu$ also satisfy the hypotheses of \eqref{eq:b-2:circle-product}. The first satisfies $F(1,c/b)=c/b$ by \eqref{eq:b-2:chain}, and they are distinct because
\begin{equation*}
 |b'|=\frac{|b||a'|}{|a|}
 \le\frac{|b||c|}{|a|}<|c|.
\end{equation*}
Applying the identity gives $c'=bF(c/b,b'/b)=b\,(c/b)(b'/b)=\mu c$. Hence $(a',b',c')=\mu T$. We used two operations to obtain $T$ and the fixed three operations in \eqref{eq:b-2:chain-copy} to obtain $\mu T$.
\end{proof}

If zero has not already appeared, the pairs $(T,T)$ and $(\mu T,T)$ now represent the starting complex numbers $1$ and $\mu$. By \Cref{lem:b-2:quotient-copies}, it remains to find an expression formed from these two values that equals zero.

To obtain zero, we will use the following lemma. We will later show how to construct the inputs $u$, $v$, and $m$ needed to apply this lemma.

\begin{lemma}
\label{lem:b-2:finite-cancellation}
Let $u,v\in\sC$ be distinct points with $\Real u=\Real v=1$, and let $m=(u+v)/2$ be their midpoint. Then $q=F(uv,m^2)$ is purely imaginary, and $F(1,q^2)=0$.
\end{lemma}
\begin{proof}
The products $uv$ and $m^2$ lie on a horizontal line. Indeed, the midpoint identity gives
\begin{equation*}
 uv-m^2=-\frac{(u-v)^2}{4}>0,
\end{equation*}
because $u-v$ is nonzero and purely imaginary. Thus $uv,m^2$ are distinct and have the same imaginary part. Their line has its closest point to zero on the imaginary axis, so $q$ is purely imaginary. Hence $q^2$ is real and nonpositive, and the line through $1,q^2$ contains zero.
\end{proof}

We use the below identity to construct the inputs $u,v$ and $m$ needed to apply \Cref{lem:b-2:finite-cancellation} and produce a zero.

\begin{lemma}
\label{lem:b-2:reflected-half}
For every nonreal $z\in\sC$ with $\Real z=1$, the value $F(1,z^2)$ is nonzero and
\begin{equation}
\label{eq:b-2:reflected-half}
 \frac1{F(1,z^2)}=\frac{1+\overline z}{2}.
\end{equation}
\end{lemma}
\begin{proof}
We compute $F(1,z^2)$ directly and show that it equals $2/(1+\overline z)$. By \eqref{eq:b-2:foot},
\begin{align*}
 F(1,z^2)
 &=1-\frac{\ip{1}{z^2-1}}{|z^2-1|^2}(z^2-1)\\
 &=1-\frac{\Real(z^2-1)}{|z^2-1|^2}(z^2-1).
\end{align*}
Here $\ip{1}{z^2-1}=\Real\left(\overline{z^2-1}\right)=\Real(z^2-1)$, since conjugation leaves the real part unchanged.

We simplify the numerator first. Because $\Real z=1$ and $z$ is nonreal, $z-1$ is nonzero and purely imaginary and its square is $-|z-1|^2$. Hence
\begin{equation*}
 \Real(z^2-1)
 =\Real\bigl((z-1)^2+2(z-1)\bigr)
 =-|z-1|^2.
\end{equation*}
For the denominator, factor $z^2-1=(z-1)(z+1)$:
\begin{equation*}
 |z^2-1|^2
 =|(z-1)(z+1)|^2
 =|z-1|^2|z+1|^2.
\end{equation*}
Substituting into the formula for $F(1,z^2)$ now gives
\begin{align*}
 F(1,z^2)
 &=1-\frac{-|z-1|^2}{|z-1|^2|z+1|^2}(z^2-1)\\
 &=1+\frac{z^2-1}{|z+1|^2}.
\end{align*}
Here we cancel $|z-1|^2$, which is nonzero because $z\ne1$.

To simplify the remaining fraction, use $|z+1|^2=(z+1)(1+\overline z)$ and $z^2-1=(z+1)(z-1)$. Bringing the two terms to a common denominator gives
\begin{align*}
 1+\frac{z^2-1}{|z+1|^2}
 &=\frac{(z+1)(1+\overline z)+(z+1)(z-1)}{(z+1)(1+\overline z)}\\
 &=\frac{(z+1)(z+\overline z)}{(z+1)(1+\overline z)}\\
 &=\frac{z+\overline z}{1+\overline z}
 =\frac2{1+\overline z}.
\end{align*}
We can cancel $z+1$ because its real part is two and hence it is nonzero, and the last equality uses $z+\overline z=2\Real z=2$. The final fraction is defined and nonzero because $\Real(1+\overline z)=2$ and hence it is nonzero. Taking its reciprocal proves the identity.
\end{proof}

We are now ready to combine the above lemmas to produce zero from any three noncollinear points in a plane.

\begin{lemma}
\label{lem:b-2:planar}
Starting from three noncollinear points in $\sC$, applying $F$ to available pairs produces zero using at most $300$ operations.
\end{lemma}
\begin{proof}
Use \Cref{lem:b-2:scaled-copy} to obtain a triple $T=(a,b,c)$ of nonzero points and a sequence of three $F$ operations producing $\mu T$, stopping if zero appears. We use \Cref{lem:b-2:quotient-copies} with $\alpha=\mu$ and $\beta=1$ to construct separate numerator and denominator triples for the expressions
\begin{equation}
\label{eq:b-2:midpoint-construction}
 u=\frac1\mu,\qquad
 v=\frac1{F(1,u^2)},\qquad
 m=\frac1{F(1,v^2)},\qquad
 q=F(uv,m^2).
\end{equation}
For $u=1/\mu$ we use the numerator triple $T$ and denominator triple $\mu T$. For $v$, \eqref{eq:b-2:quotient-fit} gives $v=\mu^2/F(\mu^2,1)$, so we construct the numerator triple $\mu^2T$ and denominator triple $F(\mu^2,1)T$ using \Cref{lem:b-2:copy-operations}. We generate other values in a similar manner as \Cref{lem:b-2:quotient-copies} allows us to do. We first check that all denominators are nonzero and $m$ is the midpoint of $u,v$.

To apply \Cref{lem:b-2:reflected-half} to $u$, we need $\Real u=1$ and $u$ nonreal. Since $\mu\ne0$, we have $u=1/\mu=\overline\mu/|\mu|^2$. Using $\Real\mu=|\mu|^2$ and the fact that $\mu$ is nonreal gives
\begin{equation*}
 \Real u=\frac{\Real\mu}{|\mu|^2}=1,
 \qquad
 \operatorname{Im}u=-\frac{\operatorname{Im}\mu}{|\mu|^2}\ne0.
\end{equation*}
Thus we have that $F(1,u^2)\ne0$ and by applying \Cref{lem:b-2:reflected-half}, we get $v=(1+\overline u)/2$. In turn, $\Real v=(1+\Real u)/2=1$ and $\operatorname{Im}v=-\operatorname{Im}u/2\ne0$. The imaginary parts of $u,v$ have opposite signs, so $u\ne v$.

We can therefore apply \Cref{lem:b-2:reflected-half} to $v$ as well. It gives $F(1,v^2)\ne0$ and
\begin{equation*}
 m=\frac{1+\overline v}{2}
 =\frac{3+u}{4}
 =\frac{2u+1+\overline u}{4}
 =\frac{u+v}{2}.
\end{equation*}
Here we substitute $v=(1+\overline u)/2$ and use $u+\overline u=2$, which follows from $\Real u=1$. This verifies that all reciprocals in \eqref{eq:b-2:midpoint-construction} are defined and that $m$ is the midpoint of the distinct numbers $u,v$, both with real part one. Applying \Cref{lem:b-2:finite-cancellation} now shows that $q$ is purely imaginary and $F(1,q^2)=0$.

Applying \Cref{lem:b-2:quotient-copies} to the expression $F(1,q^2)=0$ therefore produces a numerator triple of zero points.

For the size bound, we keep every intermediate triple and count only new operations. We give the counts in \Cref{table:b-2:operation-counts} and explain each step below.
\begin{table}[H]
\centering
\caption{Operation counts for the construction.}
\label{table:b-2:operation-counts}
\begin{tabular}{lr}
\toprule
Step & New $F$ operations (at most) \\
\midrule
Prepare $T$ & $2$ \\
Prepare $\mu T$ & $3$ \\
Represent $u$ & $0$ \\
Represent $v$ & $6$ \\
Represent $m$ & $18$ \\
Represent $uv$ & $0$ \\
Represent $m^2$ & $45$ \\
Represent $q$ & $27$ \\
Produce the numerator triple for $F(1,q^2)$ & $162$ \\
\midrule
Total & $263$ \\
\bottomrule
\end{tabular}
\end{table}

Preparing $T$ uses the two operations in \Cref{lem:b-2:scaled-copy}. The same lemma gives a sequence of three operations producing $\mu T$. Together with the available triple $T$, this represents $u=1/\mu$. By \Cref{lem:b-2:copy-operations}, running this sequence on any available triple $\eta T$ produces $\mu\eta T$ in three new operations. We will use this sequence throughout the count.

To represent $v$, first run the sequence on $\mu T$ to obtain $\mu^2T$, using three new operations. Define $D_v=F(\mu^2,1)$, so $v=\mu^2/D_v$ by \eqref{eq:b-2:quotient-fit}. The three $F$ operations in \Cref{lem:b-2:copy-operations}, applied to the triples $\mu^2T$ and $T$, produce $D_vT$. Thus representing $v$ costs $3+3=6$ new operations. Starting from $T$ alone, the complete sequence producing $D_vT$ uses $3+3+3=9$ operations.

For $m$, the same quotient rule gives
\begin{equation*}
 m=\frac1{F(1,\mu^4/D_v^2)}
 =\frac{D_v^2}{F(D_v^2,\mu^4)}.
\end{equation*}
Define $D_m=F(D_v^2,\mu^4)$. To obtain the numerator triple $D_v^2T$, run the nine-operation sequence for $D_vT$ on the stored triple $D_vT$. This also produces $\mu D_vT$ along the way, which we keep. Next, starting from the stored triple $\mu^2T$, run the three-operation sequence twice to obtain $\mu^4T$, using six new operations. Three more $F$ operations on $D_v^2T$ and $\mu^4T$ produce the denominator triple $D_mT$. The new cost for $m$ is therefore $9+6+3=18$. Including the nine operations used before this step, we have a complete sequence of $27$ operations producing $D_mT$ from $T$.

The product $uv=\mu^2/(\mu D_v)$ needs no new operations. Its numerator triple $\mu^2T$ was produced for $v$, and its denominator triple $\mu D_vT$ was kept while constructing $m$.

To represent $m^2=D_v^4/D_m^2$, start from $D_v^2T$ and run the nine-operation sequence twice, obtaining $D_v^3T$ and then $D_v^4T$. This costs $18$ new operations. Running the complete sequence for $D_mT$ on its stored output $D_mT$ gives $D_m^2T$ in $27$ more operations. Thus this step costs $18+27=45$ new operations.

For $q=F(uv,m^2)$, clearing denominators by \eqref{eq:b-2:quotient-fit} gives
\begin{equation*}
 q=\frac{F(\mu^2D_m^2,\mu D_v^5)}{\mu D_vD_m^2}.
\end{equation*}
We first construct its numerator triple. Starting from $D_m^2T$, two uses of the three-operation sequence give $\mu^2D_m^2T$, at a cost of six operations. Starting from $D_v^4T$, the nine-operation sequence gives $D_v^5T$, and three further operations give $\mu D_v^5T$. Three $F$ operations on these two triples then produce the numerator triple for $q$. Its new cost is $6+(9+3)+3=21$.

For the denominator of $q$, we reuse the triples $\mu^2D_m^2T$ and $D_m^2T$. Three $F$ operations on them produce $D_vD_m^2T$, since $F(\mu^2D_m^2,D_m^2)=D_m^2F(\mu^2,1)=D_m^2D_v$ by \eqref{eq:b-2:complex-foot}. The three-operation sequence then gives $\mu D_vD_m^2T$. Thus the denominator costs six new operations, and the total new cost for $q$ is $21+6=27$.

Finally, to represent $F(1,q^2)$, we need to run the sequences for the numerator and denominator of $q$ on their own output triples. We therefore count the length of each complete sequence starting from $T$. The numerator sequence consists of the steps for $u,v,m,m^2$, followed by the $21$ numerator operations for $q$, so its length is at most
\begin{equation*}
 3+6+18+45+21=93.
\end{equation*}
For the denominator sequence, first produce $D_mT$ in $27$ operations and then $D_m^2T$ in $27$ more. Apply the nine-operation sequence to obtain $D_vD_m^2T$, followed by three operations to obtain $\mu D_vD_m^2T$. This sequence has length $27+27+9+3=66$.

Running each of these sequences on its own output squares its coefficient by \Cref{lem:b-2:copy-operations}, so the two resulting triples represent $q^2$. This costs at most $93+66$ new operations. Three $F$ operations on those triples give the numerator triple for $F(1,q^2)$ by \eqref{eq:b-2:quotient-fit}. That triple is zero, as proved above. The final step therefore costs at most $93+66+3=162$ operations. Adding the entries of \Cref{table:b-2:operation-counts} gives $263<300$ operations in total.
\end{proof}

The construction now gives an exact fit from three predictions.
\begin{lemma}
\label{lem:b-2:pairwise-projections}
Let $g,f_1,f_2,f_3\in L^2(\gD)$ satisfy $\ip{g}{f_i}=\nrm{f_i}^2$ for $i=1,2,3$. Starting from $f_1,f_2,f_3$, at most $300$ projections of $g$ onto spans of at most two available vectors produce $P_{\spn\{f_1,f_2,f_3\}}g$.
\end{lemma}
\begin{proof}
Set $V=\spn\{f_1,f_2,f_3\}$. Replacing $g$ by $P_Vg$ preserves its inner products and projections within $V$, so assume $g\in V$. If $g=0$, all inputs are zero by the hypothesis. If $\dim V\le2$, fit from a pair spanning $V$, or one input if its dimension is one.

For $\dim V=3$, it suffices to show that $n(f_1),n(f_2),n(f_3)$ are noncollinear. \Cref{lem:b-2:planar} then reaches zero, and \Cref{lem:b-2:prediction-map} implements every step by a pairwise fit, preserving the number of operations.

If the three points lay on a line through $a$ with direction $b$, each $g+n(f_i)$ would lie in $\spn\{g+a,b\}$. By \eqref{eq:b-2:point-prediction}, so would each $f_i$, contradicting $\dim V=3$. This proves the required noncollinearity.
\end{proof}

To also account for the raw feature each agent observes, we will need to keep its contribution in every prediction. The following lemma shows that this is possible.

\begin{lemma}
\label{lem:b-2:three-predictions}
Let $f_1,f_2,f_3$ be agent predictions and let $x_j$ be a raw feature with $Y-f_i\perp x_j$ for $i=1,2,3$. At most $300$ added agents, all observing $x_j$ and having at most two parents, suffice to produce
\begin{equation*}
 P_{\spn\{x_j,f_1,f_2,f_3\}}Y.
\end{equation*}
\end{lemma}
\begin{proof}
Let $h_0=P_{\spn\{x_j\}}Y$ and $h_i=f_i-h_0$. Since $Y-f_i\perp x_j$ and $Y-h_0\perp x_j$, their difference also satisfies $h_i\perp x_j$. The input space is the orthogonal sum of $\spn\{x_j\}$ and $\spn\{h_1,h_2,h_3\}$, so the desired prediction is
\begin{equation*}
P_{\spn\{x_j,f_1,f_2,f_3\}}Y = h_0+P_{\spn\{h_1,h_2,h_3\}}Y.
\end{equation*}

The residual $Y-f_i$ is orthogonal to $f_i$ by \Cref{lem:least-squares-error}, and to $h_0\in\spn\{x_j\}$ by hypothesis. It is thus orthogonal to $h_i$, giving $\ip{Y}{h_i}=\ip{f_i}{h_i}$. Furthermore, $h_i \perp h_0$ implies $\ip{f_i}{h_i}=\nrm{h_i}^2$.

Apply \Cref{lem:b-2:pairwise-projections} with $g=Y$ and the three vectors $h_1,h_2,h_3$. It produces $P_{\spn\{h_1,h_2,h_3\}}Y$ using at most $300$ projections of $Y$, each onto the span of at most two initial vectors or vectors produced by earlier projections.

We will inductively replace each projection by an agent. We maintain that, for every available vector $v$, an agent already predicts $h_0+v$. This holds initially because the given agent predicting $f_i$ supplies $h_0+h_i=f_i$ for each $i=1,2,3$.

Consider the next projection, which uses available vectors $u,v$. By induction, agents predicting $h_0+u$ and $h_0+v$ already exist. Create an agent observing $x_j$ and receiving these two predictions. Since $h_0\in\spn\{x_j\}$, its input space equals $\spn\{x_j\}+\spn\{u,v\}$. We have $u,v \in \spn\{h_1,h_2,h_3\}$, and hence $\spn\{x_j\}\perp \spn\{u,v\}$. The new agent therefore predicts
\begin{equation}
\label{eq:b-2:pair-implementation}
 P_{\spn\{x_j,h_0+u,h_0+v\}}Y
 =h_0+P_{\spn\{u,v\}}Y.
\end{equation}
Thus the new projection also has an agent supplying its sum with $h_0$, completing the induction. Since we replaced each projection with one agent, the total number of added agents is at most $300$.
\end{proof}

The given parents need not satisfy $Y-f_i\perp x_j$, as required by \Cref{lem:b-2:three-predictions}. We show in the next lemma that with at most three extra agents we can satisfy this condition and then apply \Cref{lem:b-2:three-predictions}.

\begin{lemma}
\label{lem:b-2:adaptive-gadget}
Consider an agent observing a raw feature $x_j$ and receiving three parent predictions $f_1,f_2,f_3$. For every distribution with finite second moments, there is a gadget that reproduces this agent's prediction
\begin{equation*}
 P_{\spn\{x_j,f_1,f_2,f_3\}}Y.
\end{equation*}
The gadget uses at most $303$ added agents, each observing $x_j$ and having at most two parents. The graph and output agent may depend on the distribution.
\end{lemma}
\begin{proof}
Let $V=\spn\{x_j,f_1,f_2,f_3\}$. If $P_VY$ lies in the span of $x_j$ and any two of $f_1,f_2,f_3$, one agent using those inputs suffices. Otherwise, we construct three agents, all observing $x_j$, whose predictions $p_1,p_2,p_3$ together with $x_j$ span $V$.

Start with the fit $p_0=P_{\spn\{x_j\}}Y$. Since $p_0\ne P_VY$, at least one of $f_1,f_2,f_3$ has a nonzero inner product with $Y-p_0$. Otherwise this residual would be orthogonal to $V$. Reorder them so $\ip{Y-p_0}{f_1}\ne0$. An agent observing $x_j$ and receiving $f_1$ predicts $p_1=P_{\spn\{x_j,f_1\}}Y$. This improves on $p_0$, so its coefficient on $f_1$ is nonzero, giving $\spn\{x_j,p_1\}=\spn\{x_j,f_1\}$.

Next, $p_1$ is the fit from $x_j,f_1$, but it is not $P_VY$. Its residual must therefore have a nonzero inner product with $f_2$ or $f_3$. Order these two predictions so $\ip{Y-p_1}{f_2}\ne0$. By the span identity for $p_1$, an agent receiving $p_1,f_2$ predicts
\begin{equation*}
 p_2=P_{\spn\{x_j,p_1,f_2\}}Y
 =P_{\spn\{x_j,f_1,f_2\}}Y.
\end{equation*}
The improvement on $p_1$ forces a nonzero coefficient on $f_2$, so $\spn\{x_j,p_1,p_2\}=\spn\{x_j,f_1,f_2\}$.

Finally, $p_2\ne P_VY$ because no two of $f_1,f_2,f_3$ suffice. Its residual is orthogonal to $x_j,f_1,f_2$, so $\ip{Y-p_2}{f_3}\ne0$. An agent receiving $p_2,f_3$ therefore predicts $p_3=P_{\spn\{x_j,p_2,f_3\}}Y$ with smaller error than $p_2$. Without $f_3$, the fit would remain $p_2$, since $Y-p_2$ is orthogonal to $x_j,p_2$. Thus the coefficient on $f_3$ is nonzero, and
\begin{equation*}
 \spn\{x_j,p_1,p_2,p_3\}
 =\spn\{x_j,f_1,f_2,f_3\}=V.
\end{equation*}

All three agents observe $x_j$, so $Y-p_i\perp x_j$. By \Cref{lem:b-2:three-predictions}, at most $300$ further agents combine their predictions to produce $P_VY$. Including the three agents constructed above gives the bound of $303$.
\end{proof}

\subsection{Replacing a three-parent agent by a fixed gadget}
\label{app:b-2:fixed-gadget}

To fix the graph, we run all replacements of at most $303$ agents in parallel and combine their outputs.

\fixedtwoparentgadget*
\begin{proof}
List all DAGs with between one and $303$ added agents, each observing $x_j$ and receiving at most two predictions from earlier agents or the three external inputs. Include every choice of output agent. There are finitely many choices of parent lists and output after indexing the agents in a topological order. Let $B$ be the number of candidates.

Run these candidates in parallel and combine their outputs through a fixed balanced binary tree whose agents also observe $x_j$. Set $V=\spn\{x_j,f_1,f_2,f_3\}$ and $g=P_VY$. By \Cref{lem:b-2:adaptive-gadget}, at least one candidate predicts $g$. Every candidate and tree prediction lies in $V$, since agents take linear combinations of their inputs.

Whenever a tree agent receives $g$ from a parent, its input space contains $g$ and lies in $V$. Since $Y-g\perp V$, its fit is $g$. Following the path from the exact candidate to the root therefore proves that the output is $g$.

The candidates use at most $303B$ agents and the tree uses $B-1$, for a total of at most $304B-1$. The additional depth is at most $303+\lceil\log_2B\rceil$. Since $B$ is fixed independently of the distribution, these are universal bounds, and the entire graph and its output are fixed.
\end{proof}

\subsection{Applying the gadget to the three-parent constructions}
\label{app:b-2:bounds}

Replacing each three-parent agent by the fixed gadget preserves the predictions throughout a network. We record the size and depth bounds in the next lemma, which applies to both settings.

\begin{lemma}
\label{lem:b-2:network-replacement}
Every single-feature graph with $N$ agents, at most three parents per agent, and output depth $D$ can be replaced by a single-feature graph with $O(N)$ agents, at most two parents per agent, and output depth $O(D)$, with the same output prediction for every distribution with finite second moments. The replacement graph and allocation depend only on the original graph and allocation.
\end{lemma}
\begin{proof}
Replace each agent with three parents by the fixed gadget in \Cref{lem:b-2:fixed-gadget}, using the same raw feature throughout. Use the gadget's output wherever that agent's prediction is required, including at the network output. Leave other agents unchanged. These choices depend only on $G$ and its allocation.

In a topological order, each gadget receives the same parent predictions as the agent it replaces and therefore computes the same fit. Induction thus preserves the output prediction for every distribution. Each replacement has constant size, giving $O(N)$ agents in total. Contracting each gadget to one vertex maps any path ending at the new output to a path in $G$, which has at most $D$ agents. Each gadget contributes only a constant number of agents to the path, so the new depth is $O(D)$.
\end{proof}

\twoparentbounds*
\begin{proof}
For the oblivious designer setting with $d\ge2$, apply \Cref{lem:b-2:network-replacement} to the construction in \Cref{thm:b-3:fixed}. For $d=1$, one agent observing $x_1$ is exact. For the adaptive designer setting, apply \Cref{lem:b-2:network-replacement} to the construction in \Cref{thm:b-3:exact}.
\end{proof}

\section{Proofs for the lower bound on the number of agents}
\label{app:agent-size}

\subsection{The distribution}
\label{app:size:algebraic-independence}
\label{app:size:normalization}

We need the upper-triangular covariance entries to satisfy no nonzero polynomial relation with rational coefficients. We define this condition before proving that such entries can be chosen in the required intervals.

For $k\ge1$, a \emph{polynomial with rational coefficients} in the variables $z_1,\ldots,z_k$, denoted $p(z_1,\ldots,z_k)$, is a finite sum of terms $q\,z_1^{m_1}\cdots z_k^{m_k}$, with a rational coefficient $q$ and nonnegative integer exponents $m_1,\ldots,m_k$. After combining terms with the same powers, the polynomial is zero if every coefficient is zero. It is constant if it does not depend on any variable. For a nonzero polynomial, its \emph{degree} is the largest sum $m_1+\cdots+m_k$ among terms with nonzero coefficients. Its degree in a single variable $z_j$ is the largest exponent $m_j$ among those terms.

For polynomials $p,q$ in the same variables, we say that $p$ \emph{divides} $q$ if $q=ph$ for some polynomial $h$. A nonconstant polynomial $p$ is \emph{irreducible over $\mathbb{Q}$} if it cannot be factored into a product of two nonconstant polynomials with rational coefficients. All polynomials in these definitions must have rational coefficients.

Real numbers $a_1,\ldots,a_k$ are \emph{algebraically independent over $\mathbb{Q}$} if $p(a_1,\ldots,a_k)\ne0$ for every nonzero polynomial $p$ with rational coefficients. For example, the polynomial $z_2-z_1^2$ rules out every pair with $a_2=a_1^2$.

\begin{lemma}
\label{lem:size:algebraic-independence}
For every $d\ge2$, there is a symmetric $d\times d$ matrix $E$ with $|E_{ij}|<1/(8d)$ for all $i,j\in[d]$ such that the upper-triangular entries of $\Sigma=\frac12 I+E$ are algebraically independent over $\mathbb{Q}$: no nonzero polynomial with rational coefficients vanishes at these entries.
\end{lemma}
\begin{proof}
We construct $\Sigma$ first and then define $E=\Sigma-\frac12I$. To ensure $|E_{ij}|<1/(8d)$, we choose each diagonal entry of $\Sigma$ in the interval $(\frac12-\frac1{8d},\frac12+\frac1{8d})$ and each off-diagonal entry in $(-\frac1{8d},\frac1{8d})$.

We choose the upper-triangular entries of $\Sigma$ one at a time, in any fixed order, and fill the lower triangle by symmetry. We show that at each step, only countably many values in the permitted interval would violate algebraic independence.

Suppose $k$ entries have been chosen, with algebraically independent values $a_1,\ldots,a_k$. At the first step, $k=0$ and there are no chosen entries. Fix any nonzero polynomial $p(z_1,\ldots,z_k,z)$ with rational coefficients, and let $m$ be its degree in $z$. For each $0\le j\le m$, collect the terms in which $z$ has exponent $j$ and factor out $z^j$ and call the remaining polynomial $p_j(z_1,\ldots,z_k)$, so
\begin{equation*}
 p(z_1,\ldots,z_k,z)
 =\sum_{j=0}^m p_j(z_1,\ldots,z_k)z^j,
\end{equation*}
with each $p_j$ having rational coefficients and $p_m$ nonzero by the choice of $m$. If $k\ge1$, algebraic independence gives $p_m(a_1,\ldots,a_k)\ne0$. If $k=0$, the leading coefficient $p_m$ is a nonzero rational constant. Thus substituting the chosen values leaves a polynomial
\begin{equation*}
 p(a_1,\ldots,a_k,z)
 =\sum_{j=0}^m p_j(a_1,\ldots,a_k)z^j
\end{equation*}
of degree $m$ in one real variable $z$. It has at most $m$ real roots. Excluding these roots ensures that the next entry does not make this particular polynomial vanish.

We must exclude the roots for every such $p$. There are countably many polynomials with rational coefficients: each is specified by a finite list of rational coefficients and nonnegative integer exponents. Each polynomial excludes finitely many values, so the union of all excluded values is countable. The permitted open interval for each entry of $\Sigma$ is uncountable, so we can choose $a_{k+1}$ outside this union. Then $p(a_1,\ldots,a_{k+1})\ne0$ for every nonzero polynomial $p$ with rational coefficients, so $a_1,\ldots,a_{k+1}$ are algebraically independent. After all upper-triangular entries have been chosen, $E=\Sigma-\frac12I$ has the required properties.
\end{proof}

We next verify that the bound on the entries of $E$ makes the distribution normalized.

\begin{lemma}
\label{lem:size:distribution}
For every symmetric $E$ with $|E_{ij}|<1/(8d)$, the matrix $\Sigma=\frac12 I+E$ is positive definite. The distribution in \eqref{eq:size:distribution} has linearly independent features and satisfies
\begin{equation*}
 \nrm{x_i}^2<1,\qquad
 \frac1{3d}<w_i^*<\frac2{3d}
 \quad\text{for every }i\in[d].
\end{equation*}
It also satisfies
\begin{equation}
\label{eq:size:fixed-correlations}
 \E[xx^\top]=\Sigma,\qquad \E[xY]=\Sigma w^*=c.
\end{equation}
For every $u\in\R^d$, the linear combination $u^\top x$ of the features satisfies $\ip{u^\top x}{Y}=u^\top c$.
\end{lemma}
\begin{proof}
We first show that $\Sigma$ is positive definite, so that both $x\sim N(0,\Sigma)$ and $w^*=\Sigma^{-1}c$ are well defined. The bound on the entries of $E$ and Cauchy--Schwarz give, for every $v\in\R^d$,
\begin{equation*}
 |v^\top Ev|
 \le\sum_{i=1}^d \sum_{j=1}^d |E_{ij}|\,|v_i|\,|v_j|
 \le\frac1{8d}\left(\sum_{i=1}^d|v_i|\right)^2
 \le\frac18 v^\top v.
\end{equation*}
Since $\Sigma=I/2+E$, this implies $v^\top\Sigma v=v^\top v/2+v^\top Ev\ge3v^\top v/8>0$ whenever $v\ne0$. Hence $\Sigma$ is positive definite.

Also, $\E[xx^\top]=\Sigma$ because $x$ has mean zero and covariance $\Sigma$. Thus $\nrm{v^\top x}^2=v^\top\Sigma v>0$ for every nonzero $v$, proving that the features are linearly independent. Their second moments satisfy
\begin{equation*}
 \nrm{x_i}^2=\Sigma_{ii}
 <\frac12+\frac1{8d}<1.
\end{equation*}

To bound the coefficients of $w^*$, write $\Sigma w^*=c$ coordinatewise as
\begin{equation*}
 w_i^*=\frac1{2d}-2\sum_{j=1}^d E_{ij}w_j^*.
\end{equation*}
The entry bound gives $\sum_j|E_{ij}|<1/8$ in every row. Taking absolute values in the coordinate equation and then the maximum over $i$ yields
\begin{equation*}
 \max_i|w_i^*|
 \le\frac1{2d}+\frac14\max_i|w_i^*|,
 \qquad\text{hence}\qquad
 \max_i|w_i^*|\le\frac2{3d}.
\end{equation*}
Substituting this bound back into the same equation gives
\begin{equation*}
 \left|w_i^*-\frac1{2d}\right|
 \le2\sum_{j=1}^d|E_{ij}|\max_k|w_k^*|
 <\frac14\cdot\frac2{3d}
 =\frac1{6d}.
\end{equation*}
Thus $1/(3d)<w_i^*<2/(3d)$ for every $i$, and $\sum_i|w_i^*|<2/3<1$.

Finally, the definitions of $Y$ and $w^*$ give
\begin{equation*}
 \E[xY]=\E[xx^\top]w^*=\Sigma w^*=c.
\end{equation*}
Together with $\E[xx^\top]=\Sigma$, this proves \eqref{eq:size:fixed-correlations}. For every $u\in\R^d$, it also gives $\ip{u^\top x}{Y}=u^\top\E[xY]=u^\top c$.
\end{proof}

\subsection{Counting the input inner products}
\label{app:size:input-count}

\sizeinputcount*
\begin{proof}
Write $f_i=w_i^\top x$, and let $e_\ell$ be the $\ell$th coordinate vector. We count the equations for the raw feature paired with each parent and for pairs of different parents. We then show that these also imply the equations $w_j^\top\Delta w_j=0$.

Set the diagonal entries of $\Delta$ to zero, leaving $\binom d2$ free entries in a symmetric matrix. At every agent $A_i$ observing $x_\ell$, impose
\begin{align*}
 e_\ell^\top\Delta w_j&=0
 &&(A_j\in\Pa(A_i)),\\
 w_j^\top\Delta w_k&=0
 &&(A_j,A_k\in\Pa(A_i),\ j<k).
\end{align*}
These are homogeneous linear equations in the free entries of $\Delta$. An agent with $p$ parents contributes at most $p+\binom p2=\binom{p+1}{2}$ equations. By hypothesis, the total is less than $\binom d2$, so the solution space has positive dimension and contains a nonzero matrix $\Delta$. The equations $e_\ell^\top\Delta e_\ell=0$ already hold because the diagonal is zero. We next prove $w_j^\top\Delta w_j=0$ under the imposed equations.

Fix such a solution. We prove $w_i^\top\Delta w_i=0$ in topological order. At a source observing $x_\ell$, the vector $w_i$ is a multiple of $e_\ell$, so the claim follows from $\Delta_{\ell\ell}=0$.

Now consider an agent $A_i$ observing $x_\ell$, and assume $w_j^\top\Delta w_j=0$ for each parent $A_j$. Its prediction is a linear combination of $x_\ell$ and its parents' predictions. Since the features are linearly independent, the same linear combination relates their coefficient vectors: there are real coefficients $\alpha$ and $\beta_j$ such that
\begin{equation*}
 w_i=\alpha e_\ell+\sum_{A_j\in\Pa(A_i)}\beta_jw_j.
\end{equation*}
Substituting this expression and using the symmetry of $\Delta$ gives
\begin{align*}
 w_i^\top\Delta w_i
 &=\alpha^2\Delta_{\ell\ell}
   +2\alpha\sum_{A_j\in\Pa(A_i)}\beta_j e_\ell^\top\Delta w_j\\
 &\quad+\sum_{A_j,A_k\in\Pa(A_i)}\beta_j\beta_k w_j^\top\Delta w_k.
\end{align*}
The first term is zero because $\Delta$ has zero diagonal. Every term in the first sum is zero by the imposed raw-feature equations. In the double sum, the terms with $j\ne k$ vanish by the imposed parent equations and symmetry, and those with $j=k$ vanish by the induction hypothesis. Thus $w_i^\top\Delta w_i=0$, completing the induction.

The imposed equations already cover every pair of distinct inputs. The zero diagonal and the induction cover each input paired with itself, so \eqref{eq:size:input-direction} holds at every agent.
\end{proof}

\subsection{What exact aggregation requires}
\label{app:size:exact-inputs}
\label{app:size:polynomials}
\label{app:size:determinant}

To prove \Cref{lem:size:exact-inputs}, we will express the output coefficients as ratios of polynomials. We write $C$ for a symmetric $d\times d$ matrix whose upper-triangular entries are separate variables. We write $K(C)$ for a vector of $d$ polynomials in these variables, and $L(C)$ for a single polynomial. Evaluating them at the entries of $\Sigma$ gives the vector $K(\Sigma)$ and the number $L(\Sigma)$.

Let $\Sigma$ be chosen and fixed as in \Cref{lem:size:algebraic-independence}. We will define $K(C)$ and $L(C)$ for this $\Sigma$, requiring that $K(\Sigma)/L(\Sigma)$ gives the output coefficient vector for this distribution. Another matrix $\Sigma'$ may lead to a different choice of polynomials. We do not require $K(\Sigma')/L(\Sigma')$ to give the output coefficient vector for the distribution using $\Sigma'$.

\begin{lemma}
\label{lem:size:polynomials}
Fix $E$ satisfying the conditions in \Cref{lem:size:algebraic-independence} and define $\Sigma$ and the distribution as in \eqref{eq:size:distribution}. Fix a single-feature DAG on this distribution. There are a polynomial vector $K(C)$ and a scalar polynomial $L(C)$, with rational coefficients, such that the output coefficient vector is $K(\Sigma)/L(\Sigma)$ and $L(\Sigma)>0$. If a symmetric matrix $\Delta$ satisfies \eqref{eq:size:input-direction} at every agent, then for every $t\in\R$,
\begin{equation}
\label{eq:size:constant-denominator}
 K(\Sigma+t\Delta)=K(\Sigma),\qquad
 L(\Sigma+t\Delta)=L(\Sigma).
\end{equation}
\end{lemma}

\begin{proof}
We will construct the polynomials $K_i(C)$ and $L_i(C)$ for each agent $A_i$ in topological order. The output agent's polynomials will give the desired $K(C)$ and $L(C)$. We show that
\begin{equation*}
  w_i=\frac{K_i(\Sigma)}{L_i(\Sigma)},\qquad L_i(\Sigma)>0.
\end{equation*}
Fix any symmetric matrix $\Delta$ satisfying \eqref{eq:size:input-direction}: at each agent, $u^\top\Delta v=0$ for every pair of input coefficient vectors $u,v$, including $u=v$. We will also show that evaluating $K_i$ and $L_i$ at $\Sigma+t\Delta$ gives the same values as at $\Sigma$, for every real $t$.

For a source observing $x_\ell$, the fit uses only a multiple of $x_\ell$. By \eqref{eq:size:fixed-correlations}, we have $\E[x_\ell Y]=c_\ell$ and $\E[x_\ell^2]=\Sigma_{\ell\ell}$, so its coefficient vector is
\begin{equation*}
 w_i=\frac{\ip{x_\ell}{Y}}{\nrm{x_\ell}^2}e_\ell
 =\frac{\E[x_\ell Y]}{\E[x_\ell^2]}e_\ell
 =\frac{c_\ell}{\Sigma_{\ell\ell}}e_\ell.
\end{equation*}
Set $K_i(C)=c_\ell e_\ell$ and $L_i(C)=C_{\ell\ell}$. Their coefficients are rational because $c_\ell=1/(4d)$, and $L_i(\Sigma)=\Sigma_{\ell\ell}>0$ because $\Sigma$ is positive definite. The numerator does not depend on $C$. For the denominator, applying the condition on $\Delta$ with $u=v=e_\ell$ gives $\Delta_{\ell\ell}=0$. Hence $L_i(\Sigma+t\Delta)=\Sigma_{\ell\ell}+t\Delta_{\ell\ell}=L_i(\Sigma)$.

Now suppose the polynomials have been constructed for the parents of $A_i$, which observes $x_\ell$. Choose parents $A_{j_1},\ldots,A_{j_s}$ so that the inputs $x_\ell,f_{j_1},\ldots,f_{j_s}$ are linearly independent and
\begin{equation*}
 \spn\{x_\ell,f_{j_1},\ldots,f_{j_s}\}
 =\spn\bigl(\{x_\ell\}\cup\{f_j:A_j\in\Pa(A_i)\}\bigr).
\end{equation*}
We can keep $x_\ell$ in this selection because $\nrm{x_\ell}^2=\Sigma_{\ell\ell}>0$. If $x_\ell$ alone spans all the inputs, take $s=0$. This selection depends on $\Sigma$, but not on $\Delta$.

Since the features are linearly independent, the coefficient vectors $e_\ell,w_{j_1},\ldots,w_{j_s}$ are also linearly independent and span all coefficient vectors the agent can use. By induction, $K_j(\Sigma)=L_j(\Sigma)w_j$ with $L_j(\Sigma)>0$ for every parent. Replacing each selected $w_j$ by $K_j(\Sigma)$ therefore only multiplies that vector by a nonzero scalar, preserving both independence and the span.

Using these selected parent indices, define the $d\times(s+1)$ matrix
\begin{align*}
 B_i(C)&=\begin{bmatrix}e_\ell&K_{j_1}(C)&\cdots&K_{j_s}(C)\end{bmatrix}.
\end{align*}
Every entry of $B_i(C)$ is either $0$, $1$, or an entry of a selected $K_{j_k}(C)$, so it is a polynomial with rational coefficients.

All allowed coefficient vectors for $A_i$ are of the form $w=B_i(\Sigma)\alpha$ for some $\alpha\in\R^{s+1}$. We derive the equations for the best fit one input at a time. Name the columns of $B_i(\Sigma)$ as
\begin{equation*}
 b_0=e_\ell,\qquad b_k=K_{j_k}(\Sigma)\quad(1\le k\le s).
\end{equation*}
The column $b_0$ represents the raw feature $b_0^\top x=x_\ell$. Each other column represents a multiple of a parent prediction:
\begin{equation*}
 b_k^\top x=K_{j_k}(\Sigma)^\top x
 =L_{j_k}(\Sigma)w_{j_k}^\top x
 =L_{j_k}(\Sigma)f_{j_k}.
\end{equation*}
Now take $w=w_i$, the agent's fitted coefficient vector. Its residual $Y-w^\top x$ is orthogonal to every input, and therefore to $b_k^\top x$ for every $0\le k\le s$.

Fix one such column $b_k$. Expanding this orthogonality condition gives a scalar equation:
\begin{align*}
 0
 &=\ip{b_k^\top x}{Y-w^\top x}\\
 &=\E[(b_k^\top x)Y]-\E[(b_k^\top x)(w^\top x)]\\
 &=b_k^\top\E[xY]-b_k^\top\E[xx^\top]w\\
 &=b_k^\top c-b_k^\top\Sigma w.
\end{align*}
To obtain the third line, use $(b_k^\top x)(w^\top x)=b_k^\top xx^\top w$ and take the fixed vectors $b_k,w$ outside the expectations. The last line uses $\E[xY]=c$ and $\E[xx^\top]=\Sigma$ from \eqref{eq:size:fixed-correlations}.

We have therefore obtained $b_k^\top\Sigma w=b_k^\top c$ for every $k=0,\ldots,s$. The rows of $B_i(\Sigma)^\top$ are exactly $b_0^\top,\ldots,b_s^\top$, so these equations together say $B_i(\Sigma)^\top\Sigma w=B_i(\Sigma)^\top c$. Finally, substituting $w=B_i(\Sigma)\alpha$ yields
\begin{equation*}
 B_i(\Sigma)^\top\Sigma B_i(\Sigma)\alpha=B_i(\Sigma)^\top c.
\end{equation*}
Write $M_i(C)=B_i(C)^\top C B_i(C)$ for the matrix on the left, with $\Sigma$ replaced by $C$. At $\Sigma$ it is positive definite: for every nonzero $\alpha$,
\begin{equation*}
 \alpha^\top M_i(\Sigma)\alpha
 =(B_i(\Sigma)\alpha)^\top\Sigma(B_i(\Sigma)\alpha)>0.
\end{equation*}
The inequality holds because the columns of $B_i(\Sigma)$ are independent, so $B_i(\Sigma)\alpha\ne0$, and $\Sigma$ is positive definite. Thus $M_i(\Sigma)$ is invertible and has positive determinant.

By the definition of $M_i$, the equation for the fitted weights is $M_i(\Sigma)\alpha=B_i(\Sigma)^\top c$. Multiplying both sides on the left by $M_i(\Sigma)^{-1}$ gives
\begin{equation*}
 \alpha=M_i(\Sigma)^{-1}B_i(\Sigma)^\top c.
\end{equation*}
The entries of $\alpha$ are the weights on the columns of $B_i(\Sigma)$. To recover the coefficients on the raw features, substitute this solution into $w_i=B_i(\Sigma)\alpha$:
\begin{equation*}
 w_i=B_i(\Sigma)\alpha
 =B_i(\Sigma)M_i(\Sigma)^{-1}B_i(\Sigma)^\top c.
\end{equation*}

The inverse $M_i(C)^{-1}$ introduces division by $\det M_i(C)$, so we separate the numerator and denominator to obtain polynomials with rational coefficients.

For a square matrix $M$, let $\adj(M)$ be its adjugate, the transpose of its cofactor matrix. The identity $M\adj(M)=(\det M)I$ gives $M^{-1}=\adj(M)/\det M$ when $M$ is invertible. We therefore define
\begin{equation*}
 K_i(C)=B_i(C)\adj(M_i(C))B_i(C)^\top c,\qquad
 L_i(C)=\det M_i(C).
\end{equation*}
The determinant and every entry of the adjugate are polynomials in the matrix entries. Since the entries of $B_i(C)$ are polynomials with rational coefficients and $c=\vone_d/(4d)$ is rational, the same holds for $L_i(C)$ and each entry of $K_i(C)$. The formula for the fit now gives $w_i=K_i(\Sigma)/L_i(\Sigma)$, with $L_i(\Sigma)>0$.

Finally, replace $\Sigma$ by $\Sigma+t\Delta$ in these polynomials. By induction, $K_j(\Sigma+t\Delta)=K_j(\Sigma)$ for every parent. In particular, the selected columns $K_{j_1},\ldots,K_{j_s}$ have the same values at both matrices. The column $e_\ell$ is fixed, so $B_i(\Sigma+t\Delta)=B_i(\Sigma)$.

Every column of $B_i(\Sigma)$ is either $e_\ell$ or $L_j(\Sigma)w_j$ for a parent $A_j$. Thus each entry of $B_i(\Sigma)^\top\Delta B_i(\Sigma)$ is a scalar multiple of $u^\top\Delta v$ for input coefficient vectors $u,v$. These products are zero by \eqref{eq:size:input-direction}, so
\begin{align*}
 M_i(\Sigma+t\Delta)
 &=B_i(\Sigma)^\top(\Sigma+t\Delta)B_i(\Sigma)\\
 &=M_i(\Sigma)+tB_i(\Sigma)^\top\Delta B_i(\Sigma)
 =M_i(\Sigma).
\end{align*}
The formulas defining $K_i$ and $L_i$ use only $B_i$, $M_i$, and the fixed vector $c$, so their values are unchanged as well. These equalities hold for every real $t$, even when $\Sigma+t\Delta$ is singular, because the formulas defining the polynomials do not require its inverse. This completes the induction. The polynomials for the output agent give $K$ and $L$.
\end{proof}

A \emph{factorization into irreducibles} of a nonzero polynomial $f$ is an expression $f=a p_1\cdots p_r$, with $a$ a nonzero rational constant and each $p_i$ irreducible over $\mathbb{Q}$. Factors may repeat. For a nonzero constant $f$, we take $r=0$ and interpret the empty product as $1$.

The following standard theorem makes precise how two such factorizations can differ \citep{dummit2003abstract}.

\begin{theorem}[Unique factorization of polynomials]
\label{thm:size:unique-factorization}
Every nonzero polynomial with rational coefficients in finitely many variables has a factorization into irreducibles. If $f=a p_1\cdots p_r=b q_1\cdots q_s$ are two such factorizations, then $r=s$. After reordering the factors, for each $i$ there is a nonzero rational constant $u_i$ such that $q_i=u_i p_i$.
\end{theorem}

The following lemma is a consequence of \Cref{thm:size:unique-factorization}.

\begin{lemma}[Irreducible divisors of a product]
\label{lem:size:irreducible-product}
Let $p,f,g$ be polynomials with rational coefficients in the same variables. If $p$ is irreducible over $\mathbb{Q}$ and divides $fg$, then $p$ divides $f$ or $p$ divides $g$.
\end{lemma}

We now prove that the determinant of any symmetric matrix with separate upper-triangular variables meets the irreducibility hypothesis of \Cref{lem:size:irreducible-product}. This is also a standard result and we include a proof for completeness.

\begin{lemma}
\label{lem:size:determinant}
Let $C$ be a symmetric $d\times d$ matrix whose upper-triangular entries are separate variables. Then the polynomial $\det C$ in these variables is irreducible over the rational numbers.
\end{lemma}
\begin{proof}
We use induction on $d$. For $d=1$, the determinant is a single variable and is irreducible. For $d\ge2$, write
\begin{equation*}
 C=\begin{pmatrix}t&v^\top\\v&B\end{pmatrix}.
\end{equation*}
The coefficient of $t$ in $\det C$ is $\det B$, obtained by deleting the first row and column. Thus $\det C$ has degree one in $t$.

Suppose $\det C=pq$ for polynomials $p,q$ in the entries of $C$, with rational coefficients. Since degrees in $t$ add under multiplication, one of $p,q$ must be independent of $t$. Name the factors so that $p$ is independent of $t$. Then $q$ has degree one in $t$. Let $h_0$ and $h_1$ be its coefficient of $t$ and its constant term, respectively. Both are polynomials in the entries of $B$ and $v$, independent of $t$, and
\begin{equation*}
 \det C=p(h_0t+h_1).
\end{equation*}
The coefficient of $t$ on the right is $ph_0$. Comparing with the coefficient of $t$ in $\det C$ gives
\begin{equation*}
 \det B=ph_0.
\end{equation*}
We will apply the induction hypothesis to this factorization of $\det B$. Since $\det B=ph_0$ is independent of every entry of $v$, both $p$ and $h_0$ must also be independent of those entries: degrees in each variable add under multiplication. Thus $p$ and $h_0$ are polynomials only in the entries of $B$.

By induction, $\det B$ is irreducible, so $p$ or $h_0$ must be constant. If $p$ is constant, we are done. Otherwise, $h_0$ is a nonzero rational constant. Substituting $p=(\det B)/h_0$ into the factorization of $\det C$ gives
\begin{equation*}
 \det C=p(h_0t+h_1)=(\det B)\left(t+\frac{h_1}{h_0}\right).
\end{equation*}
This identity would force $\det C=0$ at every matrix with $\det B=0$, because division by the nonzero constant $h_0$ is always defined.

To rule this out, consider the numerical matrix
\begin{equation*}
 C_0=\begin{pmatrix}
 0&1&0\\
 1&0&0\\
 0&0&I_{d-2}
 \end{pmatrix},
\end{equation*}
using just the upper-left $2\times2$ block when $d=2$. Let $B_0$ be the lower-right $(d-1)\times(d-1)$ block of $C_0$. The first row of $B_0$ is zero, so $\det B_0=0$. But $\det C_0=-1$, from the determinant of the upper-left $2\times2$ block and the remaining identity block. This contradicts $\det C=(\det B)(t+h_1/h_0)$, whose right-hand side is zero at $C_0$. Therefore $p$ must be constant, proving that $\det C$ is irreducible.
\end{proof}

\sizeexactinputs*
\begin{proof}
Take $K,L$ from \Cref{lem:size:polynomials}. We first show that $\det C$ divides $L(C)$. We then use a singular matrix on the line $\Sigma+t\Delta$ to rule out $\Delta\ne0$.

Exact aggregation means that the output coefficient vector equals $w^*=\Sigma^{-1}c$, since the features are linearly independent. Hence
\begin{equation*}
 \frac{K(\Sigma)}{L(\Sigma)}=\Sigma^{-1}c,
\end{equation*}
and so $\Sigma K(\Sigma)=L(\Sigma)c$. Each coordinate of $CK(C)-L(C)c$ is a polynomial with rational coefficients in the upper-triangular entries of $C$, and the displayed equality says that it vanishes at $\Sigma$. Those entries of $\Sigma$ are algebraically independent by \Cref{lem:size:algebraic-independence}. Each coordinate must therefore be the zero polynomial, giving
\begin{equation*}
 CK(C)=L(C)c.
\end{equation*}
This identity holds for every symmetric $C$, including singular matrices.

Let $\widetilde C$ be obtained from $C$ by replacing its first column with $CK(C)$. Also let $C'$ be obtained from $C$ by replacing its first column with $\vone_d$. Since $CK(C)=L(C)c=(L(C)/(4d))\vone_d$, the first column of $\widetilde C$ is $L(C)/(4d)$ times the first column of $C'$, and all other columns agree. Factoring this scalar out of the first column gives
\begin{equation*}
 \det\widetilde C=\frac{L(C)}{4d}\det C'.
\end{equation*}

To compute the same determinant another way, write $C=[C_1,\ldots,C_d]$, so $C_j$ is the $j$th column, and write $(K(C))_j$ for the $j$th coordinate of $K(C)$. By definition,
\begin{equation*}
 CK(C)=\sum_{j=1}^d (K(C))_j C_j.
\end{equation*}
The first column of $\widetilde C$ is this sum, and its remaining columns are $C_2,\ldots,C_d$. Since the determinant is linear in its first column,
\begin{equation*}
 \det\widetilde C
 =\sum_{j=1}^d (K(C))_j\det[C_j,C_2,\ldots,C_d].
\end{equation*}
For $j=1$, the matrix inside the determinant is $C$. For every $j\ge2$, its first and $j$th columns are both $C_j$, so its determinant is zero. Only the $j=1$ term remains, giving
\begin{equation*}
 \det\widetilde C=(K(C))_1\det C.
\end{equation*}
Equating the two expressions for $\det\widetilde C$ and multiplying by $4d$ gives
\begin{equation*}
 4d\,(K(C))_1\det C=L(C)\det C'.
\end{equation*}
The displayed identity shows that $\det C$ divides the product $L(C)\det C'$. All three are polynomials with rational coefficients in the upper-triangular entries of $C$, and $\det C$ is irreducible by \Cref{lem:size:determinant}. By \Cref{lem:size:irreducible-product}, $\det C$ therefore divides $L(C)$ or $\det C'$.

It cannot divide $\det C'$. The first column of $C'$ consists of ones, so each term in its determinant is a product of $d-1$ entries of $C$. Thus $\det C'$ has degree at most $d-1$. It is also a nonzero polynomial: at $C=I$, subtracting the other columns of $C'$ from its first column gives $I$ without changing the determinant, so $\det C'=1$ there. In contrast, $\det C$ has degree $d$. Hence $\det C$ cannot divide $\det C'$ and must instead divide $L(C)$.

Suppose $\Delta\ne0$. We will find a real $t$ for which $\Sigma+t\Delta$ is singular. Since $\Sigma$ is positive definite, set
\begin{equation*}
 M=\Sigma^{-1/2}\Delta\Sigma^{-1/2}.
\end{equation*}
The matrix $M$ is symmetric because $\Delta$ and $\Sigma^{-1/2}$ are symmetric. It is nonzero because $\Delta=\Sigma^{1/2}M\Sigma^{1/2}$ and $\Delta\ne0$. Thus $M\ne0$ has a nonzero real eigenvalue $\lambda$ and a nonzero vector $v$ with $Mv=\lambda v$.

Set $t=-1/\lambda$, so $(I+tM)v=(1+t\lambda)v=0$. Using $\Delta=\Sigma^{1/2}M\Sigma^{1/2}$, we obtain
\begin{equation*}
 (\Sigma+t\Delta)(\Sigma^{-1/2}v)
 =\Sigma^{1/2}(I+tM)v=0.
\end{equation*}
The vector $\Sigma^{-1/2}v$ is nonzero because $\Sigma^{-1/2}$ is invertible and $v\ne0$. Hence $\Sigma+t\Delta$ has a nonzero vector in its kernel and is singular.
Since $\det C$ divides $L(C)$, we have $L(C)=\det C \cdot H(C)$ for some polynomial $H$. Since $\det (\Sigma + t\Delta) = 0$, this forces $L(\Sigma+t\Delta)=0$. But \Cref{lem:size:polynomials} gives $L(\Sigma+t\Delta)=L(\Sigma)>0$ because \eqref{eq:size:input-direction} holds. This contradiction proves $\Delta=0$.
\end{proof}

\subsection{Proof of the lower bound}
\label{app:size:lower-bound}

\begin{proof}[Proof of \Cref{thm:size:lower-bound}]
Fix $d\ge2$ and let the adversary choose the distribution from \Cref{sec:agent-size}. It is normalized and has linearly independent features by \Cref{lem:size:distribution}. Since $Y$ is a linear combination of the features, $Y=f^*$. With knowledge of this distribution, the designer may choose any graph $G=(A,E)$, single-feature allocation, and output agent for which $f_G=Y$.

Suppose, for contradiction, that $G$ violates \eqref{eq:size:parent-pairs}. By \Cref{lem:size:input-count}, there is a nonzero symmetric matrix $\Delta$ satisfying \eqref{eq:size:input-direction} at every agent. Since $G$ is exact on the chosen distribution, \Cref{lem:size:exact-inputs} gives $\Delta=0$, a contradiction. Thus $G$ satisfies \eqref{eq:size:parent-pairs}.
\end{proof}

\end{document}

%% file: math_commands.tex
\usepackage{amsmath,amsfonts,bm}

\def\eqref#1{equation~\ref{#1}}

\def\1{\bm{1}}

\def\vone{{\bm{1}}}

\DeclareMathAlphabet{\mathsfit}{\encodingdefault}{\sfdefault}{m}{sl}
\SetMathAlphabet{\mathsfit}{bold}{\encodingdefault}{\sfdefault}{bx}{n}

\def\gC{{\mathcal{C}}}
\def\gD{{\mathcal{D}}}

\def\sC{{\mathbb{C}}}

\newcommand{\E}{\mathbb{E}}

\newcommand{\R}{\mathbb{R}}

\DeclareMathOperator*{\argmax}{arg\,max}
\DeclareMathOperator*{\argmin}{arg\,min}

%% file: references.bib
@inproceedings{kearns2026networked,
  title={Networked Information Aggregation via Machine Learning},
  author={Kearns, Michael and Roth, Aaron and Ryu, Emily},
  booktitle={Proceedings of the 2026 Annual {ACM-SIAM} Symposium on Discrete Algorithms, {SODA} 2026, Vancouver, BC, Canada, January 11-14, 2026},
  pages={4799--4845},
  year={2026},
  organization={SIAM},
  doi={10.1137/1.9781611978971.173}
}

@article{bateni2026networked,
  title={Networked Information Aggregation for Binary Classification},
  author={Bateni, MohammadHossein and Hadizadeh, Zahra and Hajiaghayi, MohammadTaghi and JafariRaviz, Mahdi and Taherijam, Shayan},
  journal={CoRR},
  volume={abs/2605.01082},
  year={2026},
  doi={10.48550/ARXIV.2605.01082},
  eprinttype={arXiv},
  eprint={2605.01082}
}

@article{degroot1974reaching,
  title={Reaching a consensus},
  author={DeGroot, Morris H},
  journal={Journal of the American Statistical Association},
  volume={69},
  number={345},
  pages={118--121},
  year={1974},
  publisher={Taylor \& Francis}
}

@article{golub2010naive,
  title={Naive learning in social networks and the wisdom of crowds},
  author={Golub, Benjamin and Jackson, Matthew O},
  journal={American Economic Journal: Microeconomics},
  volume={2},
  number={1},
  pages={112--149},
  year={2010},
  publisher={American Economic Association}
}

@article{aumann1976agreeing,
  title={Agreeing to Disagree},
  author={Aumann, Robert J},
  journal={The Annals of Statistics},
  volume={4},
  number={6},
  pages={1236--1239},
  year={1976}
}

@inproceedings{collina2025tractable,
  title={Tractable agreement protocols},
  author={Collina, Natalie and Goel, Surbhi and Gupta, Varun and Roth, Aaron},
  booktitle={Proceedings of the 57th Annual ACM Symposium on Theory of Computing},
  pages={1532--1543},
  year={2025}
}

@inproceedings{collina2026collaborative,
  title={Collaborative prediction: Tractable information aggregation via agreement},
  author={Collina, Natalie and Globus-Harris, Ira and Goel, Surbhi and Gupta, Varun and Roth, Aaron and Shi, Mirah},
  booktitle={Proceedings of the 2026 Annual ACM-SIAM Symposium on Discrete Algorithms (SODA)},
  pages={4712--4798},
  year={2026},
  organization={SIAM}
}

@article{wolpert1992stacked,
  title={Stacked generalization},
  author={Wolpert, David H.},
  journal={Neural Networks},
  volume={5},
  number={2},
  pages={241--259},
  year={1992},
  doi={10.1016/S0893-6080(05)80023-1}
}

@inproceedings{kim2019multiaccuracy,
  title={Multiaccuracy: Black-Box Post-Processing for Fairness in Classification},
  author={Kim, Michael P. and Ghorbani, Amirata and Zou, James Y.},
  booktitle={Proceedings of the 2019 {AAAI/ACM} Conference on AI, Ethics, and Society, {AIES} 2019, Honolulu, HI, USA, January 27-28, 2019},
  pages={247--254},
  year={2019},
  doi={10.1145/3306618.3314287}
}

@article{yang2019federated,
  title={Federated Machine Learning: Concept and Applications},
  author={Yang, Qiang and Liu, Yang and Chen, Tianjian and Tong, Yongxin},
  journal={{ACM} Trans. Intell. Syst. Technol.},
  volume={10},
  number={2},
  pages={12:1--12:19},
  year={2019},
  doi={10.1145/3298981}
}

@article{cheng2021secureboost,
  title={SecureBoost: {A} Lossless Federated Learning Framework},
  author={Cheng, Kewei and Fan, Tao and Jin, Yilun and Liu, Yang and Chen, Tianjian and Papadopoulos, Dimitrios and Yang, Qiang},
  journal={{IEEE} Intell. Syst.},
  volume={36},
  number={6},
  pages={87--98},
  year={2021},
  doi={10.1109/MIS.2021.3082561}
}

@book{dummit2003abstract,
  title={Abstract Algebra},
  author={Dummit, D.S. and Foote, R.M.},
  isbn={9780471433347},
  lccn={2003057652},
  url={https://books.google.com/books?id=KJDBQgAACAAJ},
  year={2003},
  publisher={Wiley}
}

@inproceedings{guo2024multiagent,
  title={Large Language Model Based Multi-agents: {A} Survey of Progress and Challenges},
  author={Guo, Taicheng and Chen, Xiuying and Wang, Yaqi and Chang, Ruidi and Pei, Shichao and Chawla, Nitesh V. and Wiest, Olaf and Zhang, Xiangliang},
  booktitle={Proceedings of the Thirty-Third International Joint Conference on Artificial Intelligence, {IJCAI} 2024, Jeju, South Korea, August 3-9, 2024},
  pages={8048--8057},
  publisher={ijcai.org},
  year={2024},
  doi={10.24963/IJCAI.2024/890}
}

@article{bateni2026optimal,
  title={Optimal Rates for Agentic Networked Information Aggregation},
  author={Bateni, MohammadHossein and Hadizadeh, Zahra and Hajiaghayi, MohammadTaghi and JafariRaviz, Mahdi and Taherijam, Shayan},
  journal={arXiv preprint arXiv:2609.05318},
  year={2026}
}

@article{pal2026optimal,
  title={Optimal Lower Bounds for Networked Information Aggregation},
  author={Pal, Ambar},
  journal={arXiv preprint arXiv:2608.15472},
  year={2026}
}
